%% file: main.tex
\documentclass[lettersize,journal]{IEEEtran}

\usepackage{amsmath,amsfonts}
\usepackage{algorithm}
\usepackage{algpseudocode}
\usepackage{array}
\usepackage{textcomp}
\usepackage{stfloats}
\usepackage{url}
\usepackage{verbatim}
\usepackage{graphicx}
\usepackage{cite}
\usepackage{booktabs}
\usepackage{cellspace}
\usepackage[shortlabels]{enumitem}
\usepackage{placeins}

\usepackage{booktabs}

\input{packages.tex}

\input{math_commands.tex}
\long\def\comment#1{}

\begin{document}

\title{
Federated Low-Rank Koopman Learning for Multivariate Time-Series Anomaly Detection in IoT Systems}

\author{Tung-Anh Nguyen, Van-Phuc Bui, Anh Tuyen Le, Kim Hue Ta,\\ Minh Thuy Le, J.Andrew Zhang, and Xiaojing Huang
\thanks{T.-A Nguyen, A. T. Le, J. A. Zhang, and X. Huang are with Faculty of Engineering and IT, University of Technology Sydney, Ultimo, NSW, 2007, Australia (emails: tunganh.nguyen@uts.edu.au, anhtuyen.le@uts.edu.au, andrew.zhang@uts.edu.au, xiaojing.huang@uts.edu.au). V.-P Bui is with the Faculty of Information Technology, FPT University, Vietnam (email: phucbv5@fe.edu.vn). Kim Hue Ta and Minh Thuy Le are with School of Electrical and Electronic Engineering, Hanoi University of Science and Technology, Vietnam (email: hue.tathikim@hust.edu.vn, minhthuy.le@hust.edu.vn).}

}

\markboth{Journal of \LaTeX\ Class Files,~Vol.~14, No.~8, August~2021}%
{Shell \MakeLowercase{\textit{et al.}}: A Sample Article Using IEEEtran.cls for IEEE Journals}

\maketitle

\input{sections/abstract}


\input{sections/introduction}

\input{sections/background}
\input{sections/jiot_method}
\input{sections/jiot_convergence}

\input{sections/jiot_experiment}
\FloatBarrier

\input{sections/conclusion/conclusion}


\bibliographystyle{IEEEtran}
\bibliography{references}



%

\end{document}

%% file: packages.tex
\usepackage{microtype}
\usepackage{graphicx}
\usepackage{booktabs} 
\usepackage{hyperref}

\usepackage{amsmath}
\usepackage{amssymb}
\usepackage{mathtools}
\usepackage{amsthm}
\usepackage{enumitem}

\usepackage[capitalize,noabbrev]{cleveref}
\usepackage{hyperref}
\usepackage{xr-hyper}

\makeatletter
\newcommand*{\addFileDependency}[1]{
  \typeout{(#1)}
  \@addtofilelist{#1}
  \IfFileExists{#1}{}{\typeout{No file #1.}}
}
\makeatother

\theoremstyle{plain}
\newtheorem{theorem}{Theorem}[section]

\newtheorem{lemma}[theorem]{Lemma}

\theoremstyle{definition}

\theoremstyle{remark}
\newtheorem{remark}[theorem]{Remark}

\usepackage{hyperref}

\usepackage{url}
\usepackage{graphicx,wrapfig,lipsum} 
\usepackage{times}
\usepackage{soul}
\usepackage{url}
\usepackage[utf8]{inputenc}
\usepackage[small]{caption}
\usepackage{graphicx}
\usepackage{amsthm}
\usepackage{booktabs}
\usepackage{mwe} 
\usepackage{wrapfig}

\usepackage{setspace}
\usepackage{amsmath,amssymb,amsfonts}
\input{math_commands.tex}
\crefformat{equation}{(#2#1#3)}
\crefrangeformat{equation}{(#3#1#4) to~(#5#2#6)}
\crefmultiformat{equation}{(#2#1#3)}%
{ and~(#2#1#3)}{, (#2#1#3)}{ and~(#2#1#3)}
\crefname{assumption}{assumption}{assumptions}
\crefname{Assumption}{Assumption}{Assumptions}
\usepackage{bbm}
\usepackage[mathscr]{euscript}
\usepackage{graphicx}
\usepackage{textcomp}
\usepackage{xcolor}

\usepackage{xspace}
\usepackage{caption}
\usepackage{subcaption}
\usepackage{bm}
\usepackage{array}
\usepackage{multirow}

\makeatletter
\newcommand{\multilines}[1]{%
	\begin{tabularx}{\dimexpr\linewidth-\ALG@thistlm}[t]{@{}X@{}}
		#1
	\end{tabularx}
}
\makeatother

\usepackage[textsize=tiny]{todonotes}

\Crefname{figure}{Fig.}{Figs.}
\Crefname{table}{Table}{Tables}
\Crefname{section}{Sec.}{Secs.}

\DeclarePairedDelimiterX{\norm}[1]{\lVert}{\rVert}{#1}
\DeclarePairedDelimiterX{\abs}[1]{\lvert}{\rvert}{#1}

\DeclareSymbolFont{extraup}{U}{zavm}{m}{n}
\DeclareMathSymbol{\varheart}{\mathalpha}{extraup}{86}

\DeclareUnicodeCharacter{FB01}{fi}



%% file: math_commands.tex
\usepackage{amsmath,amsfonts,bm}

\def\eqref#1{(\ref{#1})}

\def\1{\bm{1}}

%% file: sections/abstract.tex
\begin{abstract}

Distributed IoT systems generate multivariate time-series streams for monitoring physical assets, servers, and embedded sensing platforms. Detecting abnormal temporal behavior is critical for fault diagnosis, predictive maintenance, and security. However, practical IoT anomaly detection is hindered by decentralized and non-IID data, limited bandwidth, and the constrained computation and memory of edge devices. This paper proposes FedKAD, a resource-efficient federated Koopman anomaly detection framework for distributed IoT multivariate time series. Unlike deep-learning-based anomaly detectors that require training and communicating large neural models, FedKAD learns normal temporal dynamics through lightweight sliding-window Koopman representations. Federated training is formulated as a low-rank consensus problem, where raw sensor streams and local reduced dynamics remain on device while only compact subspace variables are exchanged with the server. To optimize the shared representation under orthonormality constraints, we develop a federated Stiefel-ADMM algorithm and provide convergence and stationarity analysis under partial client participation. During inference, each client detects anomalies locally by measuring the prediction residual between observed future trajectories and the learned Koopman dynamics. Experiments on four widely used multivariate time-series anomaly detection benchmarks show that FedKAD maintains or improves detection performance compared with federated deep-learning baselines. More importantly for IoT deployment, FedKAD provides up to $2.1\times10^3$ faster training, $80\times$ lower communication, and $79\times$ lower inference latency than neural baselines, confirming its suitability for resource-constrained edge devices.

\end{abstract}

%% file: sections/introduction.tex
\section{Introduction}

Distributed Internet-of-Things (IoT) systems generate multivariate time-series (MVTS) streams for monitoring physical assets, industrial processes, servers, spacecraft subsystems, and embedded sensing platforms. Detecting abnormal temporal behavior in these streams is essential for fault diagnosis, predictive maintenance, and security monitoring \cite{chandola2009anomaly,cook2020anomaly}. Unlike conventional offline anomaly detection, IoT anomaly detection must often operate continuously at the network edge, where sensor streams are produced by geographically distributed devices with limited communication, computation, memory, and energy resources. These constraints make it difficult to rely on centralized data collection or computationally expensive models, especially when raw sensor measurements are privacy-sensitive or operationally restricted.

A fundamental challenge in IoT MVTS anomaly detection is that abnormal behavior is rarely limited to an isolated point outlier. In many practical systems, faults, cyberattacks, and degradation patterns appear as abnormal temporal evolution, sensor drift, delayed responses, changes in cross-channel dependencies, or deviations from the normal dynamical behavior of the monitored process. For example, a machine fault may first manifest as a subtle change in vibration and temperature dynamics before a large point-wise spike appears, while a compromised IoT device may exhibit abnormal temporal communication patterns rather than a single abnormal packet. Therefore, an effective detector should capture both temporal dependencies and inter-variable relationships, and should identify when the observed trajectory becomes inconsistent with learned normal dynamics.

Early time-series anomaly detection methods relied on statistical models, distance-based detectors, density estimation, or forecasting residuals \cite{chandola2009anomaly,cook2020anomaly}. These approaches are attractive because they are often interpretable and computationally lightweight. However, they usually require strong assumptions about stationarity, linearity, or simple distributional forms, and they struggle to model nonlinear dependencies in high-dimensional MVTS. Thresholding methods based on extreme value theory, such as (streaming) peaks-over-threshold, can reduce the need for manually chosen thresholds in streaming settings \cite{siffer2017pot}, but they are typically applied on top of a precomputed anomaly score and do not by themselves solve the problem of learning complex normal dynamics from distributed multivariate sensor streams. Deep-learning-based anomaly detectors have been developed to address the nonlinear and high-dimensional nature of MVTS~\cite{xu2018unsupervised,
zong2018deep, malhotra2016lstm, su2019robust,
zhao2020multivariate, deng2021graph, xu2022anomaly, tuli2022tranad}. Reconstruction-based models, including autoencoder and variational-autoencoder approaches, learn compact representations of normal data and flag samples with large reconstruction errors \cite{xu2018donut,zong2018dagmm}. Recurrent models, such as Long Short-Term Memory (LSTM) encoder-decoders and stochastic recurrent networks, further capture temporal dependencies and have been applied to multi-sensor and spacecraft telemetry anomaly detection \cite{malhotra2016lstm,hundman2018spacecraft,su2019omni}. More recent methods combine adversarial training, graph neural networks, attention mechanisms, and transformer architectures to model complex temporal and inter-sensor dependencies \cite{zhao2020mtadgat,deng2021gdn,xu2022anomalytransformer,tuli2022tranad}. These methods have achieved strong detection accuracy in centralized benchmarks. 

Despite their success, deep anomaly detectors are not always suitable for federated IoT deployment. First, many of them rely on heavy neural reconstruction, recurrent, graph, or attention modules, requiring repeated backpropagation over local windows. Second, when used in federated learning (FL), the clients must repeatedly communicate neural parameters or gradients, which can become expensive for bandwidth-limited IoT links. In addition, their memory and computation costs can be prohibitive for weak edge devices such as low-cost gateways, Raspberry Pi-class nodes, and embedded platforms. Finally, deep models are often sensitive to non-independent and non-identically distributed (non-IID) client data, which is common in IoT systems because devices operate under different environments, workloads, anomaly rates, and sensing configurations.

FL offers a natural way to train models across distributed clients without uploading raw data to a server \cite{mcmahan2017communication,kairouz2021advances}. In IoT systems, FL has been explored for privacy-preserving analytics, intrusion detection, and decentralized spatio-temporal modeling \cite{nguyen2021fliot,nguyen2019diot,meng2021crossnode}. For example, DIoT uses federated learning to aggregate device-type-specific behavior profiles for detecting compromised IoT devices \cite{nguyen2019diot}, while cross-node federated graph neural networks model decentralized spatio-temporal sensor data without direct data sharing \cite{meng2021crossnode}. However, existing FL-based approaches usually inherit the computational and communication burden of the underlying model class. When the local model is a deep neural anomaly detector, FL protects raw data but does not remove the need for local neural training and repeated communication of large model parameters. 
These limitations suggest that federated IoT anomaly detection should not simply wrap existing deep detectors in an FL protocol. Instead, the anomaly detector should be lightweight by construction, requiring low local computation, low memory, and compact communication while allowing client-specific dynamics to remain on device.

Koopman operator theory provides a principled framework for modeling nonlinear dynamical systems through linear evolution in a lifted observable space \cite{koopman1931hamiltonian,rowley2009spectral,williams2015edmd}. This perspective is well suited to MVTS because their variables are often coupled components of an underlying dynamical process rather than independent scalar streams. In many IoT monitoring systems, a change in one variable can be an early indicator of abnormal behavior and may also affect the evolution of other variables. For example, in physiological sensing, abnormal patterns in blood pressure and heart rate are more informative when analyzed jointly than when each signal is considered in isolation; similarly, in industrial monitoring, changes in vibration, temperature, and power consumption may jointly indicate machine degradation. For MVTS data,  sliding-window or delay-coordinate observables can represent a temporal segment as a dynamical state that jointly contains recent history and cross-channel dependencies. Data-driven Koopman approximations, such as dynamic mode decomposition (DMD) and optimal mode decomposition (OMD), then learn a compact linear evolution model from snapshot pairs \cite{schmid2010dmd,williams2015edmd,wynn2013optimal}. The resulting low-rank Koopman subspace captures dominant coherent modes of normal behavior, while the reduced operator predicts their temporal evolution. Anomalies are therefore detected as deviations from these learned normal dynamics: when faults, attacks, or sensor degradation disturb either the temporal evolution of individual variables or their inter-variable coupling, the Koopman prediction residual becomes large.

In this paper, we propose \emph{FedKAD}, a federated low-rank Koopman anomaly detection framework for distributed IoT MVTS. The key idea is to replace heavy federated neural reconstruction with lightweight federated low-rank dynamical modeling. Each edge device transforms its local MVTS into sliding-window snapshots and learns a Koopman-style model of normal temporal evolution. Instead of learning a full high-dimensional Koopman operator, FedKAD approximates the local dynamics using an orthonormal low-rank subspace and a private reduced operator. During inference, each device detects anomalies locally by comparing the observed future trajectory with the trajectory predicted by its learned normal dynamics.

FedKAD is designed to match the constraints of federated IoT systems. Raw time-series snapshots never leave the device. Each client keeps its local data matrices and reduced Koopman operator private, while the server coordinates only a shared low-rank subspace through consensus optimization. This design substantially reduces communication compared with federated deep-learning baselines, because clients exchange compact subspace variables rather than full neural models. At the same time, the local reduced dynamics remain client-specific, allowing FedKAD to adapt to heterogeneous local behavior while still benefiting from a shared federated representation. The resulting optimization problem is a low-rank consensus formulation with orthonormal subspace constraints. To solve it, we develop a federated Stiefel-ADMM algorithm that alternates between local subspace updates, server-side consensus aggregation, and local dual updates. The orthonormality constraint is handled directly on the Stiefel manifold using projected gradient steps and QR retraction, following the principles of optimization on matrix manifolds \cite{absil2008optimization}, while the consensus structure follows the ADMM framework for distributed optimization \cite{boyd2011distributed}. We also provide convergence and stationarity analysis under partial client participation.

The main contributions of this paper are summarized as follows:
\begin{itemize}
\item We propose FedKAD, a federated Koopman anomaly detection framework for distributed IoT MVTS, enabling each edge device to detect abnormal temporal behavior locally without uploading raw sensor streams.
\item We formulate federated anomaly detection as a low-rank Koopman consensus problem, where each client keeps its local snapshots and reduced dynamics private while only compact subspace variables are exchanged with the server.
\item We develop a federated Stiefel-ADMM algorithm for optimizing the shared orthonormal Koopman subspace under partial client participation, and provide convergence and stationarity analysis.
\item We evaluate FedKAD on four widely used MVTS anomaly detection benchmarks under non-IID federated settings and multiple evaluation protocols, motivated by recent concerns about point-adjusted time-series anomaly detection metrics \cite{kim2022rigorous}. FedKAD achieves the best F1 score on three of four datasets while demonstrating its suitability for resource-constrained IoT deployment.
\end{itemize}


%% file: sections/background.tex
\section{Background}
\input{sections/background/jiot_back_ground}

%% file: sections/background/jiot_back_ground.tex
\subsection{Federated Learning}
\label{sec:fl-background}

FL allows distributed clients to collaboratively train a model while keeping raw data local. Each client performs local optimization and sends model updates to a server, which aggregates them into a shared model for the next communication round. FedAvg \cite{mcmahan2017communication} is the standard formulation, and subsequent FL variants have enabled large-scale privacy-aware learning across domains \cite{kairouz2021advances,li2020federated}, including MVTS classification and forecasting.

Despite these benefits, FL is limited by communication and client-side computation. Large parameter updates are expensive to exchange, while edge devices often have restricted memory, energy, and compute capacity. These issues become more severe for MVTS anomaly detection, where anomalies are scarce, labels are often unavailable, and temporal patterns vary significantly across clients. Prior work reduces communication through quantization \cite{reisizadeh2020fedpaq}, compression or sparsification \cite{sattler2019robust}, and adaptive communication \cite{caldas2018expanding}. However, such techniques can introduce extra processing and tuning overhead, motivating lightweight FL designs tailored to unsupervised anomaly detection on distributed time-series streams.

\subsection{Multivariate Time-Series Anomaly Detection}
\label{sec:mtad-background}

MVTS anomaly detection aims to identify abnormal behaviors from temporally correlated and high-dimensional observations. Existing methods include recurrent models for temporal dependency modeling \cite{malhotra2015long,hundman2018detecting}, reconstruction-based autoencoders that detect deviations from learned normal patterns \cite{zong2018deep,su2019robustly}, and GAN-based approaches that model complex data distributions through adversarial learning \cite{li2019mad,geiger2020tadgan}. More recent methods use graph neural networks to capture inter-variable dependencies \cite{deng2021graph,zhao2020multivariate} or Transformers to model long-range temporal interactions via self-attention \cite{xu2022anomaly,zhou2021informer}.

While these models improve detection capacity, they often rely on computationally heavy architectures, centralized training, or careful optimization. Such requirements are difficult to satisfy in edge and IoT settings, where data are distributed across devices, anomalies are rare and unlabeled, and clients have limited memory, energy, and computation. This motivates lightweight methods that can capture temporal structure while remaining efficient under distributed and resource-constrained deployment.

\input{sections/background/koopman_theory}

\subsection{Optimal Mode Decomposition}
\label{sec:omd}
Based on the finite-dimensional linear approximation in
Eq.~\eqref{eq:finite_koopman}, the goal is to estimate a linear operator
$K\in\mathbb{R}^{d\times d}$ that advances raw data snapshots. Given a
sequence of raw data vectors
$x_1,x_2,\ldots,x_{M+1}$, where $x_t\in\mathbb{R}^{d}$ and $M$ is the
number of one-step training pairs, we form the paired snapshot matrices
\begin{equation}
\label{eq:omd_snapshot_x}
    X =
    \big[
    x_1,x_2,\ldots,x_M
    \big]
    \in\mathbb{R}^{d\times M},\nonumber
\end{equation}
\begin{equation}
\label{eq:omd_snapshot_y}
    Y =
    \big[
    x_2,x_3,\ldots,x_{M+1}
    \big]
    \in\mathbb{R}^{d\times M}. \nonumber
\end{equation}
Here, $X,Y\in\mathbb{R}^{d\times M}$ contain consecutive raw data
snapshots. A full least-squares estimate of $K$ can be expensive and may
overfit when $d$ is large. OMD addresses this by approximating the transition
matrix with a low-rank operator of the form
\begin{equation}
\label{eq:omd_low_rank_operator}
    K \approx P Q P^\top ,\nonumber
\end{equation}
where $P\in\mathbb{R}^{d\times r}$ has orthonormal columns,
$Q\in\mathbb{R}^{r\times r}$, and $r\ll d$. The OMD objective is therefore
\begin{equation}
\label{eq:omd_objective}
    \min_{P,Q}
    \left\|
    Y-PQP^\top X
    \right\|_F^2,
    \qquad
    \mathrm{s.t.}\quad P^\top P=I .
\end{equation}
Here, $P$ spans an $r$-dimensional subspace of the raw data space, while
$Q$ models the linear dynamics within that reduced subspace. The projection
$P^\top$ maps each raw snapshot to reduced coordinates, $Q$ advances the
reduced state, and $P$ maps the prediction back to the original raw data
space. Hence, $PQP^\top\in\mathbb{R}^{d\times d}$ provides a rank-$r$
approximation of the finite-dimensional transition operator in
Eq.~\eqref{eq:finite_koopman}.

\subsection{ADMM-based Koopman Learning.}
The Alternating Direction Method of Multipliers (ADMM) is a widely used
optimization framework for solving structured problems by decomposing them
into smaller subproblems coordinated through dual variables
\cite{boyd2011distributed}. It is particularly well suited to distributed
optimization and statistical learning because local subproblems can be solved
independently while a global consensus is enforced through augmented
Lagrangian updates. This decompositional property makes ADMM attractive for
federated learning, where clients collaboratively learn a shared model without
exchanging raw data.

In this work, we exploit ADMM to develop a privacy-preserving federated
Koopman learning framework for multivariate time-series anomaly detection in
IoT systems. Each client locally estimates a low-rank Koopman model from its
own normal time-series snapshots, while the server enforces consensus among
client-side Koopman subspaces through ADMM-based aggregation. The local
dynamics model follows the OMD principle of jointly identifying a low-rank
subspace and the corresponding reduced linear dynamics
\cite{wynn2013optimal}. Unlike centralized Koopman learning, the proposed
formulation keeps raw sensor data on local devices and communicates only
model-related variables, which is consistent with the privacy-preserving
motivation of federated anomaly detection for MVTS
\cite{zhang2021federated,zhu2022deep,liu2022fedtadbench}. The learned
Koopman dynamics are then used to predict normal system evolution, and
anomalies are detected from large prediction residuals.

%% file: sections/background/koopman_theory.tex
\subsection{Koopman-Based Anomaly Detection} 

Let $x_t\in\mathcal{M}\subset\mathbb{R}^{n}$ be the state of a discrete-time
nonlinear system evolving as $x_{t+1}=\varphi(x_t)$. Koopman theory studies the
evolution of observables $\Phi$ rather than the nonlinear map $\varphi$ itself.
The Koopman operator $\mathcal{K}$ acts linearly on observables as
\begin{equation}
\label{eq:general_koopman}
\mathcal{K}\Phi(x_t)=\Phi(\varphi(x_t))=\Phi(x_{t+1}).
\end{equation}

Since $\mathcal{K}$ is generally infinite-dimensional, practical methods use a
finite observable representation $\Phi(x_t)\in\mathbb{R}^{D}$ and approximate
\begin{equation}
\label{eq:finite_koopman}
\Phi(x_{t+1})\approx K\Phi(x_t),
\end{equation}
where $K\in\mathbb{R}^{D\times D}$ is a finite-dimensional Koopman operator.

Data-driven Koopman methods, including DMD, OMD, and their variants, approximate nonlinear
dynamics from snapshot pairs through a finite-dimensional linear operator
\cite{schmid2010dmd,wynn2013optimal}. These methods are interpretable, but their accuracy depends on whether the chosen observable
space is sufficiently expressive. Extended DMD (EDMD) improves expressiveness
by lifting states into nonlinear observables~\cite{williams2015edmd}, while
deep Koopman methods learn the observable map and operator jointly
\cite{lusch2018deep}. However, these approaches often require carefully
designed dictionaries, high-dimensional lifted spaces, or additional neural
training.

Koopman-based anomaly detection uses this linearized dynamical representation
to learn normal system evolution and detect anomalies through deviations from
the learned dynamics. Koopman operator frameworks have been used for
time-series modeling and anomaly detection, while
recent Koopman predictors have been applied to anomaly detection in complex IoT
systems with time-series data~\cite{fu2024deep}. Koopman-based residual
models have also been used for fault detection and isolation in nonlinear
dynamical systems~\cite{bakhtiaridoust2022model}. These studies show that
Koopman representations are useful for structured and interpretable residual-based anomaly scoring. However, existing Koopman anomaly methods are mostly centralized, assuming that data are pooled before training.

%% file: sections/jiot_method.tex
\input{sections/jiot_problem_formulation}
\input{sections/jiot_stiefield}
\input{sections/jiot_anomaly_detection}

%% file: sections/jiot_problem_formulation.tex
\section{System Model and Problem Formulation}
\label{sec:fedkad_problem_formulation}

\begin{figure}[!t]
    \centering
    \includegraphics[width=0.5\textwidth]{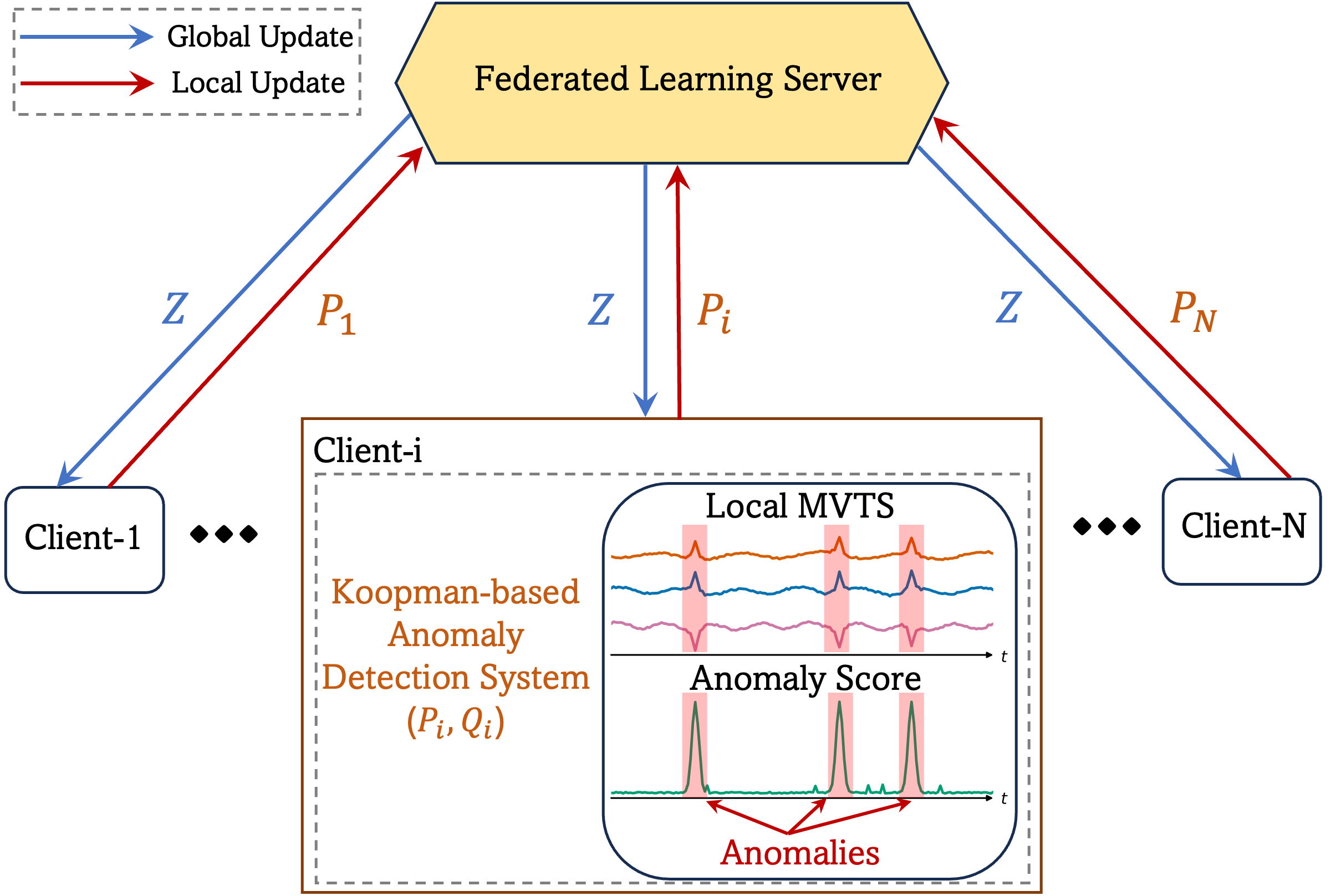}
    \caption{System model of \textsc{FedKAD} for federated IoT multivariate time-series anomaly detection.}
    \label{fig:fedks_stationary}
    \vspace{-10pt}
\end{figure}

\subsection{Federated IoT Time-Series Anomaly Detection}
\label{sec:fedkad_iot_system_model}

As illustrated in \Cref{fig:fedks_stationary}, we consider a federated IoT
monitoring system consisting of $N$ edge devices coordinated by a central
server. Client $i\in\{1,\ldots,N\}$ observes a local MVTS $x_i(t)\in\mathbb{R}^{n}$, where $n$ is the number of monitored
variables. The objective is to collaboratively learn an anomaly detector
for temporal system behavior across clients, while keeping raw
time-series data local.

In multivariate IoT systems, anomalies rarely appear as isolated
point-wise outliers~\cite{malhotra2015long,hundman2018detecting}.
Instead, they often arise from abnormal temporal evolution, delayed
effects between signals, or changes in cross-variable
dependencies~\cite{deng2021graph,zhao2020multivariate}. Therefore,
anomaly detection requires modeling not only the current observation, but
also the temporal context around it.

To capture local spatio-temporal dependencies, each client converts its
raw stream into a sliding-window state. For a window length $W$, we define
\begin{equation}
\label{eq:fedkad_window}
    z_i(t)
    =
    \big[
        x_i(t-W+1)^\top,\ldots,x_i(t)^\top
    \big]^\top
    \in\mathbb{R}^{d},
    \qquad d=nW .
\end{equation}
This delay-coordinate representation is widely used in time-series
anomaly detection, where reconstruction or prediction is performed over
local temporal contexts~\cite{tuli2022tranad}. By
flattening the window into $z_i(t)$, the model operates jointly over the
$W$ time steps and the $n$ monitored variables, thereby capturing both
temporal dependencies across time and cross-variable dependencies within
the system.

\subsection{Problem Formulation}
\label{sec:fedkad_problem}

Given the local windowed states $\{z_i(t)\}$, the goal of federated
time-series anomaly detection is to learn the normal temporal evolution
of distributed MVTS without transmitting raw client
data to the server. In general, each client uses a prediction model to
estimate the next windowed state from the current one. Let
\begin{equation}
\label{eq:fedkad_general_prediction}
    \hat z_i(t+1)
    =
    g_i\big(z_i(t);\Theta,\theta_i\big),\nonumber
\end{equation}
where $\Theta$ denotes the global parameters shared across clients and
$\theta_i$ denotes client-specific parameters. The prediction deviation is
then used as the anomaly score:
\begin{equation}
\label{eq:fedkad_general_score}
    s_i(t)
    =
    \left\|
        z_i(t+1)-\hat z_i(t+1)
    \right\|_2 .\nonumber
\end{equation}
A large value of $s_i(t)$ indicates that the observed temporal evolution
cannot be well explained by the learned normal dynamics, and is therefore
more likely to be anomalous.

The general federated learning objective can be written as
\begin{equation}
\label{eq:fedkad_general_fl_problem}
\begin{aligned}
    \min_{\Theta,\{\theta_i\}_{i=1}^{N}}\quad
    &\sum_{i=1}^{N}
    \sum_{t=W}^{T_i-1}
    \ell_i
    \left(
        z_i(t+1),
        g_i\big(z_i(t);\Theta,\theta_i\big)
    \right).
\end{aligned}
\end{equation}
where $\ell_i(\cdot)$ measures the prediction error. This formulation
allows the federation to learn common normal temporal patterns across
clients, while preserving local differences caused by non-IID data~\cite{collins2021exploiting}.

Motivated by the ability of Koopman theory to represent temporal
evolution, we reformulate
\eqref{eq:fedkad_general_fl_problem} under the Koopman theory in
\eqref{eq:general_koopman} as follows:
\begin{equation}
\label{eq:fedkad_general_koopman_problem}
\begin{aligned}
    \min_{\mathcal{K}}\quad
    &\sum_{i=1}^{N}
    \sum_{t=W}^{T_i-1}
    \left\|
        \Phi\big(z_i(t+1)\big)
        -
        \mathcal{K}\Phi\big(z_i(t)\big)
    \right\|_2^2 .
\end{aligned}
\end{equation}
However, directly solving \eqref{eq:fedkad_general_koopman_problem} is not well suited to federated IoT
settings. The lifted Koopman operator can be high-dimensional, making it
costly to store, compute, and communicate on resource-constrained edge
devices. Moreover, a single full operator may fail to capture
heterogeneous client dynamics under non-IID data. These limitations
motivate \textsc{FedKAD}, which uses a communication- and memory-efficient
low-rank shared subspace while retaining client-specific reduced dynamics.

\section{Proposed FedKAD Framework} \label{sec:fedkad_method}
To solve the federated prediction problem in
\eqref{eq:fedkad_general_koopman_problem}, \textsc{FedKAD} combines a common
nonlinear observable map with low-rank Koopman modeling. The method
consists of three main steps. First, each client maps its local
sliding-window states into a common observable space. Second, the
resulting Koopman dynamics are factorized into a shared low-rank
subspace and a client-specific reduced operator. Finally, the shared
subspace is learned through federated consensus optimization, while
the local snapshots and reduced dynamics remain on device.

\subsection{Koopman Representation for Temporal Dynamics}
\label{sec:fedkad_koopman_model}


To model nonlinear temporal dynamics, we use a Koopman-inspired
representation. Koopman theory represents nonlinear state evolution as
linear evolution in an observable space
\cite{koopman1931hamiltonian,kaiser2021data}. Since the Koopman operator is
generally infinite-dimensional, data-driven Koopman methods approximate it
using finite-dimensional observables and a linear operator learned from
snapshot pairs~\cite{schmid2010dmd,wynn2013optimal,williams2015edmd}. Thus,
for a lifting map $\Phi(\cdot)$, the local dynamics of client $i$ are
modeled as
\begin{equation}
\label{eq:fedkad_koopman_phi}
    \Phi(z_i(t+1))
    \approx
    K_i\Phi(z_i(t)), \nonumber
\end{equation}
where $K_i$ is a finite-dimensional lifted Koopman operator.

We construct $\Phi(\cdot)$ using a fixed nonlinear lifting. Classical DMD
and OMD fit a linear transition directly in the observed state space, which
is efficient but restrictive for nonlinear multivariate dynamics. Inspired
by EDMD~\cite{williams2015data}, we augment the windowed state with fixed
nonlinear observables:
\begin{equation}
\label{eq:fedkad_lift}
    \Phi(z_i(t))
    =
    \begin{bmatrix}
    \underbrace{z_i(t)}_{\text{Windowed state}} \\
    \underbrace{
    \tanh\!\left(W_{\mathrm{lift}}z_i(t)+b\right)
    }_{\text{Nonlinear observables}}
    \end{bmatrix}
    \in\mathbb{R}^{D},
    \qquad
    D=d+d_{\mathrm{lift}} .
\end{equation}
Here, $W_{\mathrm{lift}}\in\mathbb{R}^{d_{\mathrm{lift}}\times d}$ and
$b\in\mathbb{R}^{d_{\mathrm{lift}}}$ are fixed random parameters shared by
all clients. The first block preserves the original windowed state, while
the second block adds $d_{\mathrm{lift}}$ nonlinear random observables.
The $\tanh$ nonlinearity is motivated by universal approximation theory:
\emph{finite linear combinations of affine projections followed by sigmoidal
activations can approximate continuous functions on compact domains to
arbitrary accuracy}~\cite{cybenko1989approximation}. This fixed lifting
enriches the Koopman observable space without hand-designing a large
dictionary or training a neural lifting network. Since all clients use the
same window length and lifting map, their lifted states share a common
dimension and semantic structure, which makes federated subspace alignment
well defined.

For client $i$, the lifted training pairs are formed by stacking
consecutive lifted states column-wise:
\begin{equation}
\label{eq:fedkad_snapshots}
    X_i=\big[\Phi(z_i(t))\big]_{t\in\mathcal{T}_i},
    Y_i=\big[\Phi(z_i(t+1))\big]_{t\in\mathcal{T}_i},
    X_i,Y_i\in\mathbb{R}^{D\times M_i},
    \nonumber
\end{equation}
where $M_i$ is the number of valid one-step training pairs. For a local
sequence of length $T_i$, we use
\begin{equation}
\label{eq:fedkad_valid_pairs}
    \mathcal{T}_i=\{W,\ldots,T_i-1\},
    \qquad
    M_i=T_i-W .
    \nonumber
\end{equation}

\subsection{Low-Rank Local Koopman Learning Problem Formulation}
\label{sec:fedkad_low_rank_omd}
The finite-dimensional Koopman model in \eqref{eq:finite_koopman}
requires a lifted operator $K_i\in\mathbb{R}^{D\times D}$, which faces several limitations in federated IoT settings. Its storage, computation, and
communication costs scale quadratically with the lifted dimension, and a
single full operator is poorly suited to heterogeneous client dynamics
under non-IID data. These considerations motivate an OMD-inspired
low-rank formulation, which separates the Koopman subspace from the
reduced temporal operator. Accordingly, we adopt the following low-rank
factorization~\cite{wynn2013optimal}:
\begin{equation}
\label{eq:fedkad_low_rank_prediction}
    \Phi(z_i(t+1))
    \approx
    P_iQ_iP_i^\top \Phi(z_i(t)),
    \nonumber
\end{equation}
where
\begin{equation}
\label{eq:fedkad_pq_dimensions}
    P_i\in\mathbb{R}^{D\times r},
    \qquad
    Q_i\in\mathbb{R}^{r\times r},
    \qquad
    r\ll D .
    \nonumber
\end{equation}
Here, $P_i$ spans an $r$-dimensional Koopman subspace in the lifted
observable space, while $Q_i$ describes the reduced linear evolution within
that subspace. The factorization reduces the number of parameters from
$O(D^2)$ for a full lifted operator to $O(Dr+r^2)$, making the local model
more suitable for resource-constrained IoT devices.

Using the lifted snapshot matrices $(X_i,Y_i)$, client $i$ learns its local
low-rank Koopman model by solving the lifted OMD objective

\begin{equation}
\label{eq:fedkad_local_omd}
\begin{aligned}
    \min_{P_i,Q_i}\quad
    &\left\|
        Y_i-P_iQ_iP_i^\top X_i
    \right\|_F^2, \\
    \text{s.t.}\quad
    &P_i^\top P_i=I .
\end{aligned}
\end{equation}


\subsection{Federated Koopman Learning for IoT Systems}
\label{sec:fedkad_distributed_formulation}

The low-rank factorization in \eqref{eq:fedkad_low_rank_prediction} separates the Koopman subspace $P_i$ from the reduced dynamics $Q_i$. This separation aligns with shared-representation personalized federated learning, where a global representation is shared across clients while local components capture client-specific variation~\cite{collins2021exploiting}. Motivated by this view, \textsc{FedKAD} federates the low-rank Koopman subspace rather than raw data or full lifted operators. Each client keeps its snapshot matrices $(X_i,Y_i)$ and reduced operator $Q_i$ local, while the server coordinates a shared subspace through a consensus variable $Z\in\mathbb{R}^{D\times r}$. The distributed objective is
\begin{equation}
\label{eq:fedkad_consensus}
    \min_{\{P_i,Q_i\},Z}\;
    \sum_{i=1}^{N} f_i(P_i,Q_i),
\end{equation}
with the local loss
\begin{equation}
\label{eq:fedkad_local_loss}
    f_i(P_i,Q_i)
    :=
    \big\|Y_i-P_iQ_iP_i^\top X_i\big\|_F^2, \nonumber
\end{equation}
subject to
\begin{equation}
\label{eq:fedkad_consensus_constraints}
    P_i=Z,
    \qquad
    P_i^\top P_i=I,
    \qquad i=1,\ldots,N.\nonumber
\end{equation}
The first constraint forces the clients to agree on a shared low-rank
subspace $Z$, while each $Q_i$ stays private and tracks the local
dynamics of client $i$.

Since the optimal $Q$ depends on the chosen subspace, the solution of
the inner least-squares problem in $Q_i$ is denoted
$\hat{Q}_i(P_i)$. Setting $\partial f_i/\partial Q_i$ to zero gives the
normal equations
\begin{equation}
\label{eq:fedkad_Q_normal}
    \big(P_i^\top X_iX_i^\top P_i\big)\,\hat{Q}_i(P_i)^\top
    =
    P_i^\top X_iY_i^\top P_i.\nonumber
\end{equation}
Whenever the $r\times r$ Gram matrix
$G_i(P_i):=P_i^\top X_iX_i^\top P_i$ is invertible, this yields the
closed-form solution
\begin{equation}
\label{eq:fedkad_Q_hat}
    \hat{Q}_i(P_i)
    =
    P_i^\top Y_iX_i^\top P_i\,
    G_i(P_i)^{-1}.
\end{equation}
The matrix $G_i(P_i)$ is not guaranteed to be invertible, however:
$X_iX_i^\top$ has rank at most $\min(D,M_i)$, and when the local
snapshots are limited or when the columns of $P_i$ fall in a nearly
degenerate direction of $X_i$, $G_i(P_i)$ becomes singular or
severely ill-conditioned. In such cases, we fall back to a standard
Tikhonov regularization, replacing $G_i(P_i)^{-1}$ in
\eqref{eq:fedkad_Q_hat} by $\big(G_i(P_i)+\lambda I\big)^{-1}$ with
$\lambda>0$, which keeps the update well defined while preserving the
single $r\times r$ linear-system cost.

Because \(\hat{Q}_i(P_i)\) can be computed locally from \(P_i\), the
variables \(Q_i\) can be eliminated, yielding
\begin{equation}
\label{eq:fedkad_outer}
\begin{aligned}
    \min_{\{P_i\},Z}\quad
    &\sum_{i=1}^{N}
    \big\|
        Y_i-P_i\hat{Q}_i(P_i)P_i^\top X_i
    \big\|_F^2 \\
    \mathrm{s.t.}\quad
    &P_i=Z,\qquad
    P_i^\top P_i=I,\qquad i=1,\ldots,N .
\end{aligned}
\end{equation}

\subsection{ADMM-Based Federated Optimization}
\label{sec:fedkad_admm}
The linear consensus constraint $P_i=Z$ in~\eqref{eq:fedkad_outer} is well suited to ADMM~\cite{boyd2011distributed}. ADMM alternates between local primal
updates, a global consensus update, and local dual updates, progressively
aligning the client subspaces \(P_i\) with the global subspace \(Z\).
The orthonormality constraint \(P_i^\top P_i=I\) is handled directly in
the local \(P_i\) update; hence, the only dual variable is associated
with the linear consensus constraint. Let
\(\Lambda_i\in\mathbb{R}^{D\times r}\) denote the dual variable for
client \(i\), and let \(\rho>0\) be the ADMM penalty parameter. The
augmented Lagrangian of \eqref{eq:fedkad_outer} is
\begin{equation}
\small
\label{eq:fedkad_lagrangian}
\begin{split}
    \mathcal{L}
    =
    \sum_{i=1}^{N}\!\Big[
    &f_i\!\big(P_i,\hat{Q}_i(P_i)\big)
    +\langle \Lambda_i,P_i-Z\rangle_F +\tfrac{\rho}{2}\,\big\|P_i-Z\big\|_F^2
    \Big],
\end{split}
\end{equation}
where \(\langle\cdot,\cdot\rangle_F\) denotes the Frobenius inner
product. 

\subsubsection{Primal Update (Local Updates)}
At communication round \(\ell\), each client first refits its local
reduced Koopman operator using the current subspace:
\begin{equation}
\label{eq:fedkad_local_Q}
    \hat{Q}_i^{\ell+1}
    =
    \hat{Q}_i\!\big(P_i^\ell\big). \nonumber
\end{equation}
Then, client \(i\) updates its local subspace by minimizing the
augmented Lagrangian over the orthonormal manifold:
\begin{equation}
\label{eq:fedkad_local_P}
    P_i^{\ell+1}
    =
    \mathop{\mathrm{argmin}}\limits_{P_i^\top P_i=I}
    \Big[f_i\big(P_i,\hat{Q}_i^{\ell+1}\big)
    +
    \tfrac{\rho}{2}
    \big\|
        P_i-Z^\ell+\tfrac{1}{\rho}\Lambda_i^\ell
    \big\|_F^2\Big] . \nonumber
\end{equation}

\subsubsection{Consensus Update (Global Update)}
The server aggregates the updated local subspaces and dual variables:
\begin{equation}
\label{eq:fedkad_consensus_update_full}
    Z^{\ell+1}
    =
    \frac{1}{N}\sum_{i=1}^{N}
    \Big(
        P_i^{\ell+1}
        +
        \tfrac{1}{\rho}\Lambda_i^\ell
    \Big).
\end{equation}

\subsubsection{Dual Update (Local Updates)}
Each client then updates its dual variable locally:
\begin{equation}
\label{eq:fedkad_dual_update}
    \Lambda_i^{\ell+1}
    =
    \Lambda_i^\ell
    +
    \rho
    \big(
        P_i^{\ell+1}
        -
        Z^{\ell+1}
    \big).
\end{equation}

Summing \eqref{eq:fedkad_dual_update} over all clients and substituting
\eqref{eq:fedkad_consensus_update_full} gives
\begin{equation}
\label{eq:fedkad_consensus_update}
    Z^{\ell+1}
    =
    \frac{1}{N}
    \sum_{i=1}^{N}
    P_i^{\ell+1}.
\end{equation}
This simplified consensus step substantially reduces communication
overhead in the federated setting. Instead of exchanging raw snapshots
\((X_i,Y_i)\) or full Koopman operators, each client communicates only
the compact subspace matrix \(P_i\in\mathbb{R}^{D\times r}\), with
\(r\ll D\). The data matrices and the local reduced Koopman operator
\(\hat{Q}_i\) remain stored and updated entirely on device.

When only a subset
\(\mathcal{S}^\ell\subseteq\{1,\ldots,N\}\) of clients is available at
communication round \(\ell\), the updates are restricted to
\(i\in\mathcal{S}^\ell\). In this case, the consensus rule retains its
full ADMM form:
\begin{equation}
\label{eq:fedkad_consensus_update_partial}
    Z^{\ell+1}
    =
    \frac{1}{|\mathcal{S}^\ell|}
    \sum_{i\in\mathcal{S}^\ell}
    \Big(
        P_i^{\ell+1}
        +
        \tfrac{1}{\rho}\Lambda_i^\ell
    \Big),
\end{equation}
because the \eqref{eq:fedkad_consensus_update} does not generally hold when some
client dual variables are frozen. The convergence analysis in
Sec.~\ref{sec:fedkad_convergence} covers this partial-participation
setting.

%% file: sections/jiot_stiefield.tex
\subsection{Federated Koopman Learning on the Stiefel Manifold}
\label{sec:fedkad_training_algorithm}

The constraint \(P_i^\top P_i=I\) in \eqref{eq:fedkad_outer} places the
local subspace variable \(P_i\) on the Stiefel manifold
\[
    \mathrm{St}(D,r)
    =
    \{P\in\mathbb{R}^{D\times r}:P^\top P=I\}.
\]
Rather than adding a penalty for this constraint, \textsc{FedKAD}
updates \(P_i\) directly on the manifold. At communication round \(\ell\), after computing
\(\hat{Q}_i^{\ell+1}\), client \(i\) approximately solves
\begin{equation}
\label{eq:fedkad_train_P_subproblem}
    P_i^{\ell+1}
    =
    \mathop{\mathrm{argmin}}_{P\in\mathrm{St}(D,r)}
    \mathcal{F}_i^\ell(P),
\end{equation}
where
\begin{equation}
\label{eq:fedkad_train_local_objective}
    \mathcal{F}_i^\ell(P)
    :=
    f_i\big(P,\hat{Q}_i^{\ell+1}\big)
    +
    \frac{\rho}{2}
    \left\|
        P-Z^\ell+\frac{1}{\rho}\Lambda_i^\ell
    \right\|_F^2
        -
    \tfrac{1}{2\rho}\|\Lambda_i\|_F^2.
\end{equation}
To approximate \eqref{eq:fedkad_train_P_subproblem}, each client
performs projected gradient descent on \(\mathrm{St}(D,r)\). Starting
from the current local subspace \(P_i^{\ell,0}=P_i^\ell\), the client
applies \(C\) inner steps. For \(c=0,\ldots,C-1\), the Euclidean
gradient of \(\mathcal{F}_i^\ell\) is first projected onto a feasible
tangent direction,
\begin{equation}
\label{eq:fedkad_train_projected_grad}
    G_i^{\ell,c}
    =
    \big(I-P_i^{\ell,c}P_i^{\ell,c\top}\big)
    \nabla_P\mathcal{F}_i^\ell(P_i^{\ell,c}).
\end{equation}
The client then takes a gradient step and retracts the result back to
the Stiefel manifold using QR decomposition:
\begin{equation}
\label{eq:fedkad_train_retraction}
    P_i^{\ell,c+1}
    =
    \operatorname{QR}
    \left(
        P_i^{\ell,c}-\eta G_i^{\ell,c}
    \right),
\end{equation}
where \(\eta\) is the local step size.
The retraction keeps every inner iterate feasible, i.e.,
\(P_i^{\ell,c\top}P_i^{\ell,c}=I\). After \(C\) inner steps, the final
inner iterate is used as the local ADMM update:
\begin{equation}
\label{eq:fedkad_train_inner_final}
    P_i^{\ell+1}=P_i^{\ell,C}. \nonumber
\end{equation}

\begin{algorithm}[t]
\small
\caption{Federated Stiefel ADMM for \textsc{FedKAD}}
\label{alg:fedkad_training}
\begin{algorithmic}[1]
\small
    \State Server initializes the \(Z^0\) and $d_{\mathrm{lift}}$ lifting functions randomly
    \State Each client initializes \(P_i^0\) with
    \(P_i^{0\top}P_i^0=I\) and \(\Lambda_i^0=0\)
    \For{\(\ell=0,\ldots,L-1\)} \Comment{Global rounds}
        \State Server broadcasts \(Z^\ell\) to the clients
        \For{client \(i=1,\ldots,N\) in parallel}
            \State Refit
            \(\hat{Q}_i^{\ell+1}=\hat{Q}_i(P_i^\ell)\)
            by Eq.~\eqref{eq:fedkad_Q_hat}
            \State Set \(P_i^{\ell,0}=P_i^\ell\)
            \For{\(c=0,\ldots,C-1\)}
            \Comment{Local rounds}
                \State Compute \(G_i^{\ell,c}\) by
                Eq.~\eqref{eq:fedkad_train_projected_grad}
                \State Retract \(P_i^{\ell,c+1}\) by
                Eq.~\eqref{eq:fedkad_train_retraction}
            \EndFor
            \State Set \(P_i^{\ell+1}=P_i^{\ell,C}\)
            and send \(P_i^{\ell+1}\) to the server
        \EndFor
        \State Server computes \(Z^{\ell+1}\) by
        Eq.~\eqref{eq:fedkad_consensus_update_partial}
        \State Server broadcasts \(Z^{\ell+1}\) to the clients
        \For{client \(i=1,\ldots,N\) in parallel}
            \State Update \(\Lambda_i^{\ell+1}\) by
            Eq.~\eqref{eq:fedkad_dual_update}
        \EndFor
    \EndFor
\end{algorithmic}
\end{algorithm}

The complete update sequence is summarized in
Algorithm~\ref{alg:fedkad_training}.

%% file: sections/jiot_anomaly_detection.tex
\subsection{Federated Koopman Anomaly Detection}
After federated training in Algorithm~\ref{alg:fedkad_training}, each client
$i$ stores global Koopman subspace $Z$ and
local operator $\hat Q_i$.
\label{sec:fedkad_anomaly_detection}
\paragraph{Prediction.}
At each valid time $t$, client $i$ forms the windowed state $z_i(t)$ as in
Eq.~\eqref{eq:fedkad_window} and maps it to the lifted observable
$\Phi(z_i(t))$. Given a prediction horizon $H\geq1$, the
$h$-step prediction in the lifted space is
\begin{equation}
\label{eq:fedkad_ad_prediction}
    \widehat{\Phi}_i(t+h)
    =
    Z\hat Q_i^{\,h}Z^\top \Phi(z_i(t)),
    \qquad
    h=1,\ldots,H . \nonumber
\end{equation}
\paragraph{Anomaly score.}
Once the future observations become available, client $i$ constructs the
realized lifted states $\Phi(z_i(t+h))$ and compares them with
the predictions. The anomaly score is the average prediction residual
\begin{equation}
\label{eq:fedkad_ad_score}
    s_i(t)
    =
    \frac{1}{H}
    \sum_{h=1}^{H}
    \left\|
        \Phi(z_i(t+h))-\widehat{\Phi}_i(t+h)
    \right\|_2^2 . \nonumber
\end{equation}
A large $s_i(t)$ indicates that the observed future evolution is poorly
explained by the learned normal Koopman dynamics, and is therefore treated as
an anomalous event.



%% file: sections/jiot_convergence.tex
\section{\textsc{FedKAD}: Convergence Analysis}
\label{sec:fedkad_convergence}

We analyze the ADMM iterates defined in
\Cref{sec:fedkad_admm}. Let \(\ell\) denote the communication round and
let \(\mathcal{S}^{\ell}\subseteq\{1,\ldots,N\}\) denote the subset of
clients participating in round \(\ell\). Let \(\mathcal{L}^{\ell}\)
denote the augmented Lagrangian in \eqref{eq:fedkad_lagrangian}
evaluated at \((\{P_i^\ell\},Z^\ell,\{\Lambda_i^\ell\})\), after each
participating client has refitted
\(\hat Q_i^{\ell+1}=\hat Q_i(P_i^\ell)\). We use
\(\nabla_{\mathrm{St}}\) to denote the Riemannian gradient with respect
to \(P_i\) on the Stiefel manifold.

\textbf{Assumptions.}
We impose the following assumptions:
\begin{itemize}
    \item A1: For each client $i$, the local objective $f_i(P_i)$ has an
    $L_i$-Lipschitz Riemannian gradient along the generated Stiefel
    iterates:
    \begin{equation}
    \label{eq:fedkad_lipschitz_grad}
        \left\|
        \nabla_{\mathrm{St}} f_i(P_i^{\ell+1})
        -
        \nabla_{\mathrm{St}} f_i(P_i^\ell)
        \right\|_F
        \leq
        L_i
        \left\|
        P_i^{\ell+1}-P_i^\ell
        \right\|_F .\nonumber
    \end{equation}

    \item A2: For all $i$, the penalty parameter $\rho$ is chosen large
    enough such that the local $P_i$-subproblem
    \eqref{eq:fedkad_train_local_objective} is strongly convex with modulus
    $\mu_i(\rho)$.
    \item A3: The augmented Lagrangian is bounded from below along the generated iterates:
    \begin{equation}
    \label{eq:fedkad_lower_bound}
    \inf_{} \mathcal{L}^{\ell}
    >
    -\infty .
    \end{equation}
    \item A4: The local $P_i$-subproblems in \eqref{eq:fedkad_train_P_subproblem} are solved inexactly with summable errors. Specifically, there exists a nonnegative sequence ${\varepsilon_\ell}$ such that
    \begin{equation}
    \label{eq:fedkad_inner_error}
        \varepsilon_\ell
        :=
        \sum_{i\in\mathcal{S}^{\ell}}
        \left[
            \mathcal{F}_i^\ell(P_i^{\ell+1})
            - \min_{P\in\mathrm{St}(D,r)}
            \mathcal{F}_i^\ell(P)
        \right], \sum_{\ell=0}^{\infty}\varepsilon_\ell<\infty.\nonumber
    \end{equation}

    \item A5: The cumulative objective variation $\Delta_Q$ induced by the reduced Koopman-operator updates is finite.
    \begin{equation}
    \footnotesize
    \label{eq:fedkad_delta_Q}
    \Delta_Q^\ell
    :=
    \sum_{i=1}^{N}
    \left[
    f_i(P_i^{\ell+1},\hat Q_i^{\ell+1})
    -
    f_i(P_i^{\ell+1},\hat Q_i^{\ell})
    \right],
    \;
    \sum_{\ell=0}^{\infty}\Delta_Q^\ell<\infty .\nonumber
    \end{equation}
\end{itemize}

\begin{remark} Assumptions A1--A4 follow standard conditions for ADMM-type methods with smooth objectives and inexact local updates. Assumption A5 is FedKAD-specific and ensures that the objective variation from updating the reduced Koopman operators remains summable, preserving the asymptotic descent argument. \end{remark}

To establish the convergence guarantee of \textsc{FedKAD}, we analyze the generated sequence
$\big(\{P_i^\ell\}, Z^\ell, \{\Lambda_i^\ell\}\big)$.
The analysis proceeds in four steps. First, we bound the successive difference of
the dual variables in terms of the change in the local subspace variables. Next, we
derive a perturbed descent bound for the augmented Lagrangian. We then show that the
augmented Lagrangian sequence converges. Finally, we use these results to prove that
the consensus residuals vanish and that the generated sequence converges to a
stationary point. The empirical convergence behavior is further evaluated in~\ref{par:convergence-guarantee}.

\subsection{Bound on the Successive Difference of Dual Variables}

We first establish an upper bound for the successive difference of the
dual variables $\Lambda_i$ in terms of the local subspace variables $P_i$.

\begin{lemma}
\label{lemma:fedkad_dual_diff}
Suppose Assumption A1 holds. Then, for each participating client
$i\in\mathcal{S}^{\ell}$,
\begin{equation}
\label{eq:fedkad_dual_diff}
    \|\Lambda_i^{\ell+1}-\Lambda_i^\ell\|_F^2
    \leq
    L_i^2
    \|P_i^{\ell+1}-P_i^\ell\|_F^2 .
\end{equation}
\end{lemma}

\textit{Proof.}
For each participating client $i\in\mathcal{S}^{\ell}$, the optimality
condition of the local $P_i$-subproblem gives
\begin{equation}
\label{eq:fedkad_local_opt_condition}
    \nabla_{\mathrm{St}} f_i(P_i^{\ell+1})
    +
    \Lambda_i^\ell
    +
    \rho(P_i^{\ell+1}-Z^{\ell+1})
    =
    0 .\nonumber
\end{equation}
Using the dual update
\[
    \Lambda_i^{\ell+1}
    =
    \Lambda_i^\ell
    +
    \rho(P_i^{\ell+1}-Z^{\ell+1}),
\]
we obtain
\begin{equation}
\label{eq:fedkad_dual_gradient_relation}
    \Lambda_i^{\ell+1}
    =
    -
    \nabla_{\mathrm{St}} f_i(P_i^{\ell+1}) .\nonumber
\end{equation}
Similarly, at the previous local iterate,
\[
    \Lambda_i^{\ell}
    =
    -
    \nabla_{\mathrm{St}} f_i(P_i^{\ell}) .
\]
Therefore,
\[
\begin{split}
    \|\Lambda_i^{\ell+1}-\Lambda_i^\ell\|_F
    &=
    \left\|
    \nabla_{\mathrm{St}} f_i(P_i^{\ell+1})
    -
    \nabla_{\mathrm{St}} f_i(P_i^\ell)
    \right\|_F  \\
    &\leq
    L_i
    \|P_i^{\ell+1}-P_i^\ell\|_F ,
\end{split}
\]
where the last inequality follows from Assumption A1. Squaring both sides
completes the proof.
\qed

\subsection{Bound for the Augmented Lagrangian}

Next, we show that the augmented Lagrangian decreases sufficiently up to
the perturbation terms caused by inexact local optimization and Koopman
operator refitting.

\begin{lemma}
\label{lemma:fedkad_descent}
Suppose Assumptions A1--A5 hold. Then the augmented Lagrangian satisfies
\begin{equation}
\label{eq:fedkad_sufficient_decrease}
\begin{split}
\mathcal{L}^{\ell+1}-\mathcal{L}^{\ell}
\leq
&-
\frac{|\mathcal{S}^{\ell}|\rho}{2}
\|Z^{\ell+1}-Z^\ell\|_F^2 \\
&+
\sum_{i\in\mathcal{S}^{\ell}}
\left(
    \frac{L_i^2}{\rho}
    -
    \frac{\mu_i(\rho)}{2}
\right)
\|P_i^{\ell+1}-P_i^\ell\|_F^2 \\
&+
\varepsilon_\ell
+
\Delta_Q^\ell .\nonumber
\end{split}
\end{equation}
Moreover, if $\rho$ is chosen such that
\begin{equation}
\label{eq:fedkad_rho_condition}
    \rho\mu_i(\rho)\geq 2L_i^2.\nonumber
\end{equation}
then there exist constants $a,b>0$ such that
\begin{equation}
\label{eq:fedkad_descent_compact}
\begin{split}
\mathcal{L}^{\ell+1}-\mathcal{L}^{\ell}
\leq
&-
a
\sum_{i\in\mathcal{S}^{\ell}}
\|P_i^{\ell+1}-P_i^\ell\|_F^2
-
b
\|Z^{\ell+1}-Z^\ell\|_F^2  \\
&+
\varepsilon_\ell
+
\Delta_Q^\ell .
\end{split}
\end{equation}
\end{lemma}

\textit{Proof.}
For one communication round $\ell$, we split the successive difference of
the augmented Lagrangian into four terms:
\begin{align}
&\mathcal{L}^{\ell+1}-\mathcal{L}^{\ell} \nonumber\\
&=
\underbrace{
\mathcal{L}^{\ell}
(\{P_i^{\ell+1}\},Z^\ell,\{\Lambda_i^\ell\})
-
\mathcal{L}^{\ell}
(\{P_i^{\ell}\},Z^\ell,\{\Lambda_i^\ell\})
}_{A:\ \text{local }P_i\text{-update}}
\nonumber\\
&\quad+
\underbrace{
\mathcal{L}^{\ell}
(\{P_i^{\ell+1}\},Z^{\ell+1},\{\Lambda_i^\ell\})
-
\mathcal{L}^{\ell}
(\{P_i^{\ell+1}\},Z^\ell,\{\Lambda_i^\ell\})
}_{B:\ Z\text{-update}}
\nonumber\\
&\quad+
\underbrace{
\mathcal{L}^{\ell}
(\{P_i^{\ell+1}\},Z^{\ell+1},\{\Lambda_i^{\ell+1}\})
-
\mathcal{L}^{\ell}
(\{P_i^{\ell+1}\},Z^{\ell+1},\{\Lambda_i^\ell\})
}_{C:\ \text{dual update}}
\nonumber\\
&\quad+
\underbrace{
\mathcal{L}^{\ell+1}
-
\mathcal{L}^{\ell}
(\{P_i^{\ell+1}\},Z^{\ell+1},\{\Lambda_i^{\ell+1}\})
}_{D:\ \text{Koopman refitting perturbation}} . \nonumber
\label{eq:fedkad_lagrangian_split}
\end{align}
Here, term $A$ corresponds to the local $P_i$-updates with
$Z^\ell$ and $\Lambda_i^\ell$ fixed; term $B$ corresponds to the global
consensus $Z$-update; term $C$ corresponds to the dual-variable update;
and term $D$ captures the perturbation caused by refitting the reduced
Koopman operators.

\textbf{Bounding Term A.}
By Assumption A2, the local $P_i$-subproblem is strongly convex with
modulus $\mu_i(\rho)$. Since the local subproblem is solved inexactly, the
error is captured by $\varepsilon_\ell$ in Assumption A4. Therefore,
\begin{equation}
\label{eq:fedkad_term_A}
    A
    \leq
    -
    \sum_{i\in\mathcal{S}^{\ell}}
    \frac{\mu_i(\rho)}{2}
    \|P_i^{\ell+1}-P_i^\ell\|_F^2
    +
    \varepsilon_\ell .
\end{equation}

\textbf{Bounding Term B.}
With $\{P_i^{\ell+1}\}$ and $\{\Lambda_i^\ell\}$ fixed, the $Z$-update is
a quadratic minimization with strong-convexity modulus
$|\mathcal{S}^{\ell}|\rho$. Hence,
\begin{equation}
\label{eq:fedkad_term_B}
    B
    \leq
    -
    \frac{|\mathcal{S}^{\ell}|\rho}{2}
    \|Z^{\ell+1}-Z^\ell\|_F^2 .
\end{equation}

\textbf{Bounding Term C.}
The dual update changes only the linear dual term of the augmented
Lagrangian. Thus,
\begin{equation}
\label{eq:fedkad_term_C}
    C
    =
    \frac{1}{\rho}
    \sum_{i\in\mathcal{S}^{\ell}}
    \|\Lambda_i^{\ell+1}-\Lambda_i^\ell\|_F^2 .
\end{equation}
Using Lemma~\ref{lemma:fedkad_dual_diff}, we obtain
\begin{equation}
\label{eq:fedkad_term_C_bound}
    C
    \leq
    \sum_{i\in\mathcal{S}^{\ell}}
    \frac{L_i^2}{\rho}
    \|P_i^{\ell+1}-P_i^\ell\|_F^2 .
\end{equation}

\textbf{Bounding Term D.}
Term $D$ is the perturbation introduced by refitting the local reduced
Koopman operators. By Assumption A5,
\begin{equation}
\label{eq:fedkad_term_D}
    D
    =
    \Delta_Q^\ell .
\end{equation}

\textbf{Sufficient decrease.}
Combining \eqref{eq:fedkad_term_A}--\eqref{eq:fedkad_term_D}, we obtain
\begin{equation}
\label{eq:fedkad_sufficient_decrease}
\begin{split}
\mathcal{L}^{\ell+1}-\mathcal{L}^{\ell}
\leq
&-
\frac{|\mathcal{S}^{\ell}|\rho}{2}
\|Z^{\ell+1}-Z^\ell\|_F^2 \\
&+
\sum_{i\in\mathcal{S}^{\ell}}
\left(
    \frac{L_i^2}{\rho}
    -
    \frac{\mu_i(\rho)}{2}
\right)
\|P_i^{\ell+1}-P_i^\ell\|_F^2 \\
&+
\varepsilon_\ell
+
\Delta_Q^\ell . \nonumber
\end{split}
\end{equation}
This implies that the augmented Lagrangian has sufficient descent up to
the summable perturbation terms $\varepsilon_\ell$ and $\Delta_Q^\ell$ if
the following condition is satisfied:
\begin{equation}
\label{eq:fedkad_rho_condition}
    \rho\mu_i(\rho)\geq 2L_i^2 \nonumber
\end{equation}

\begin{remark}
Lemma~\ref{lemma:fedkad_descent} shows that the augmented Lagrangian has a
sufficient descent property up to two summable perturbation terms:
$\varepsilon_\ell$, caused by inexact local optimization, and
$\Delta_Q^\ell$, caused by refitting the reduced Koopman operators.
\end{remark}

\subsection{Convergence of the Augmented Lagrangian}

Here, we combine Lemma~\ref{lemma:fedkad_dual_diff} and
Lemma~\ref{lemma:fedkad_descent} to establish the convergence of the
augmented Lagrangian.

\begin{theorem}
\label{thm:fedkad_lagrangian_convergence}
Suppose Assumptions A1--A5 hold. Further suppose that after $L$
communication rounds, every client has participated at least once. Then,
by choosing $\rho$ sufficiently large, the augmented Lagrangian sequence
$\{\mathcal{L}^{\ell}\}$ converges to a finite value. Moreover,
\[
    \|Z^{\ell+1}-Z^\ell\|_F\to0,
    \qquad
    \|P_i^{\ell+1}-P_i^\ell\|_F\to0,
    \quad \forall i,
\]
and
\[
    \|P_i^{\ell+1}-Z^{\ell+1}\|_F\to0,
    \quad \forall i .
\]
\end{theorem}

\textit{Proof.}
Summing \eqref{eq:fedkad_descent_compact} from $\ell=0$ to $L-1$ yields
\begin{equation}
\label{eq:fedkad_telescoping}
\begin{split}
\mathcal{L}^{L}-\mathcal{L}^{0}
\leq
&-
a
\sum_{\ell=0}^{L-1}
\sum_{i\in\mathcal{S}^{\ell}}
\|P_i^{\ell+1}-P_i^\ell\|_F^2  \\
&-
b
\sum_{\ell=0}^{L-1}
\|Z^{\ell+1}-Z^\ell\|_F^2
+
\sum_{\ell=0}^{L-1}\varepsilon_\ell
+
\sum_{\ell=0}^{L-1}\Delta_Q^\ell .
\end{split}
\end{equation}
By Assumption A3, $\mathcal{L}^{L}$ is bounded from below. By
Assumptions A4 and A5,
\[
    \sum_{\ell=0}^{\infty}\varepsilon_\ell<\infty,
    \qquad
    \sum_{\ell=0}^{\infty}\Delta_Q^\ell<\infty .
\]
Letting $L\to\infty$ in \eqref{eq:fedkad_telescoping}, we obtain
\[
    \sum_{\ell=0}^{\infty}
    \sum_{i\in\mathcal{S}^{\ell}}
    \|P_i^{\ell+1}-P_i^\ell\|_F^2
    <
    \infty,
    \qquad
    \sum_{\ell=0}^{\infty}
    \|Z^{\ell+1}-Z^\ell\|_F^2
    <
    \infty .
\]
Therefore,
\[
    \|Z^{\ell+1}-Z^\ell\|_F\to0,
    \qquad
    \|P_i^{\ell+1}-P_i^\ell\|_F\to0
\]
along the rounds in which client $i$ participates. Since every client
participates at least once within any $L$ consecutive communication
rounds, the latter limit holds for all clients.

From the dual update,
\[
    P_i^{\ell+1}-Z^{\ell+1}
    =
    \rho^{-1}
    \left(
        \Lambda_i^{\ell+1}-\Lambda_i^\ell
    \right).
\]
Combining this with Lemma~\ref{lemma:fedkad_dual_diff}, we have
\[
    \|\Lambda_i^{\ell+1}-\Lambda_i^\ell\|_F\to0 .
\]
Therefore,
\[
    \|P_i^{\ell+1}-Z^{\ell+1}\|_F\to0 .
\]
Finally, \eqref{eq:fedkad_descent_compact} shows that all possible
increases of $\mathcal{L}^{\ell}$ are controlled by the summable
sequences $\{\varepsilon_\ell\}$ and $\{\Delta_Q^\ell\}$. Since
$\mathcal{L}^{\ell}$ is bounded from below, the sequence
$\{\mathcal{L}^{\ell}\}$ converges to a finite value.
\qed

\begin{remark}
Theorem~\ref{thm:fedkad_lagrangian_convergence} establishes that the
augmented Lagrangian converges and that the consensus constraint is
asymptotically satisfied, i.e.,
$\|P_i^{\ell+1}-Z^{\ell+1}\|_F\to0$ for all clients.
\end{remark}

\subsection{Convergence to a Stationary Point}

In the following theorem, we show that every limit point of the sequence
generated by \textsc{FedKAD} satisfies the first-order stationarity
conditions.

\begin{theorem}
\label{thm:fedkad_stationarity}
Suppose the conditions in Theorem~\ref{thm:fedkad_lagrangian_convergence}
hold and the local $P_i$-subproblems are solved to first-order
stationarity. Then every limit point
$(\{P_i^\star\},Z^\star,\{\Lambda_i^\star\})$ of the sequence generated by
\textsc{FedKAD} is a stationary solution. That is, for all
$i=1,\ldots,N$,
\begin{align}
    P_i^\star &= Z^\star,
    \label{eq:fedkad_stationary_consensus}
    \\
    \nabla_{\mathrm{St}}
    \left[
        f_i(P_i^\star,Q_i^\star)
        +
        \langle \Lambda_i^\star,
        P_i^\star-Z^\star\rangle_F
    \right]
    &=0 .
    \label{eq:fedkad_stationary_gradient}
\end{align}
Moreover,
\begin{equation}
\label{eq:fedkad_Q_stationary}
    Q_i^\star
    \in
    \arg\min_{Q_i}
    f_i(P_i^\star,Q_i).\nonumber
\end{equation}
If $P_i^{\star\top}X_iX_i^\top P_i^\star$ is nonsingular, then
\begin{equation}
\label{eq:fedkad_Q_stationary_closed_form}
    Q_i^\star
    =
    P_i^{\star\top}Y_iX_i^\top P_i^\star
    \left(
        P_i^{\star\top}X_iX_i^\top P_i^\star
    \right)^{-1}.\nonumber
\end{equation}
\end{theorem}

\textit{Proof.}
Let $(\{P_i^\star\},Z^\star,\{\Lambda_i^\star\})$ be any limit point and
consider a convergent subsequence. From
Theorem~\ref{thm:fedkad_lagrangian_convergence},
\[
    \|P_i^{\ell+1}-Z^{\ell+1}\|_F\to0 .
\]
Taking the limit along the convergent subsequence gives
$P_i^\star=Z^\star$, which proves
\eqref{eq:fedkad_stationary_consensus}.

Next, since the local $P_i$-subproblem is solved to first-order
stationarity on the Stiefel manifold, we have
\[
\nabla_{\mathrm{St}}
\left[
f_i(P_i^{\ell+1},Q_i^{\ell+1})
+
\langle \Lambda_i^\ell,
P_i^{\ell+1}-Z^\ell\rangle_F
+
\frac{\rho}{2}
\|P_i^{\ell+1}-Z^\ell\|_F^2
\right]
\to0 .
\]
Moreover,
\[
P_i^{\ell+1}-Z^\ell
=
(P_i^{\ell+1}-Z^{\ell+1})
+
(Z^{\ell+1}-Z^\ell)
\to0 .
\]
Thus, the gradient contribution of the quadratic consensus penalty
vanishes in the limit. Passing to the limit gives
\[
\nabla_{\mathrm{St}}
\left[
f_i(P_i^\star,Q_i^\star)
+
\langle \Lambda_i^\star,
P_i^\star-Z^\star\rangle_F
\right]
=0,
\]
which proves \eqref{eq:fedkad_stationary_gradient}.

Finally, for each fixed $P_i^{\ell+1}$, the reduced Koopman operator
$Q_i^{\ell+1}$ is computed as a minimizer of the local least-squares
objective. By continuity,
\[
    Q_i^\star
    \in
    \arg\min_{Q_i}
    f_i(P_i^\star,Q_i).
\]
If $P_i^{\star\top}X_iX_i^\top P_i^\star$ is nonsingular, the minimizer is
unique. Solving the normal equations gives
\[
    Q_i^\star
    =
    P_i^{\star\top}Y_iX_i^\top P_i^\star
    \left(
        P_i^{\star\top}X_iX_i^\top P_i^\star
    \right)^{-1}.
\]
Therefore, every limit point satisfies the stated stationarity
conditions.
\qed

%% file: sections/jiot_experiment.tex
\section{Experimental Results}
\label{sec:experiments}


We evaluate \textsc{FedKAD} on four widely used MVTS anomaly-detection benchmarks summarised in
\Cref{tab:dataset-stats}: (1)~\textbf{Pool Server Metrics
(PSM)}~\cite{abdulaal2021practical}, a $25$-dimensional dataset from eBay's
pooled server resources with $132{,}481$ training and $87{,}841$
testing samples and a $27.75\%$ anomaly ratio; (2)~\textbf{Server
Machine Dataset (SMD)}~\cite{su2019robust}, five weeks of
resource-utilisation traces from $28$ servers, each with $38$ metrics
(CPU, memory, network) and a $4.16\%$ anomaly ratio; (3)~\textbf{Soil
Moisture Active Passive (SMAP)} and (4)~\textbf{Mars Science
Laboratory (MSL)}~\cite{hundman2018detecting}, NASA spacecraft
telemetry with $54$ and $27$ entities and $25$ and $55$ channels
respectively (anomaly ratios $12.85\%$ and $10.53\%$).
\Cref{tab:dataset-stats} lists the per-dataset train/test counts,
anomaly ratio, channel count, and the number of FL clients induced by
the partitioning scheme described in \Cref{sec:fl-setup}.

\begin{table}[!hbp]
\centering
\caption{Dataset statistics. \emph{NS}: number of channels per
entity; \emph{NN}: number of FL clients (one per entity for
SMD/SMAP/MSL; Dirichlet-partitioned for PSM).}
\label{tab:dataset-stats}
\footnotesize
\setlength{\tabcolsep}{3pt}
\begin{tabular}{@{}lrrrrr@{}}
\toprule
Dataset & Train & Test & Anom.\,(\%) & NS & NN \\
\midrule
SMD  & 708\,405 & 708\,420 &  4.16 & 38 & 28 \\
PSM  & 132\,481 &  87\,841 & 27.75 & 25 & 24 \\
SMAP & 135\,183 & 427\,617 & 12.85 & 25 & 54 \\
MSL  &  58\,317 &  73\,729 & 10.53 & 55 & 27 \\
\bottomrule
\end{tabular}
\end{table}
\begin{figure*}[!tbp]
\centering
\includegraphics[width=\linewidth,height=0.32\textheight,keepaspectratio]{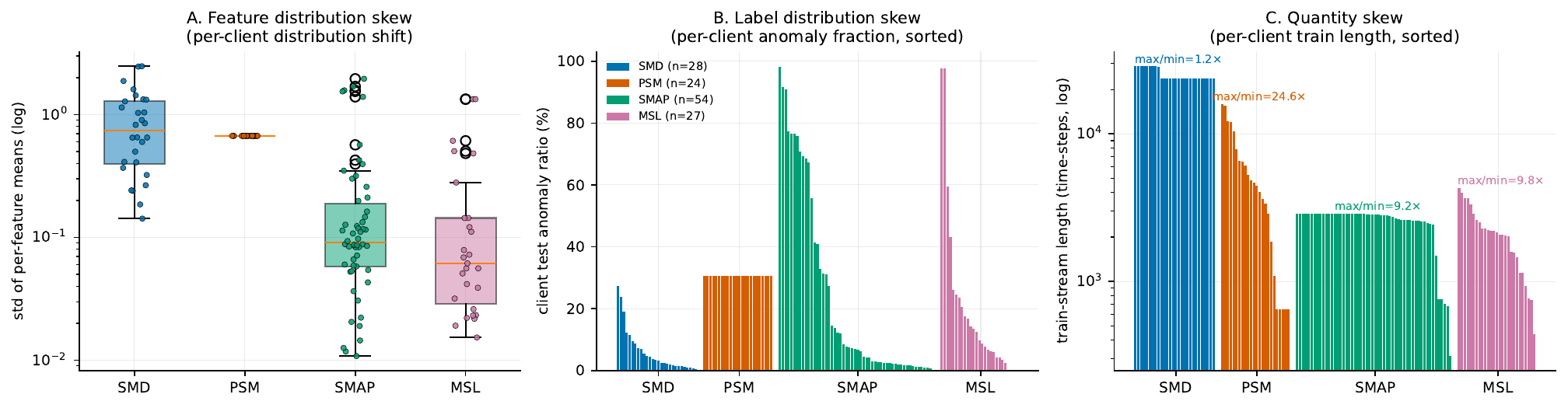}
\caption{Non-IID heterogeneity across the four benchmarks, covering feature distribution skew, label distribution skew, and quantity skew.}
\label{fig:heterogeneity}
\end{figure*}
\subsection{Federated learning settings}
\label{sec:fl-setup}

\subsubsection{Dataset}
For SMD, SMAP and MSL we use the \emph{natural per-entity partition}:
one physical entity (server or spacecraft subsystem) becomes one FL
client, so its train and test streams remain entity-specific and
disjoint and no cross-entity leakage is possible. This yields
$28$, $54$ and $27$ clients respectively. PSM provides a single
shared stream, so we induce non-IID-ness by Dirichlet partitioning
following \cite{mcmahan2017communication}: the
$132{,}481$-sample training stream is sliced into $24$ contiguous
chunks of mean length $5{,}520$ which are then assigned to clients
via Dirichlet$(\alpha{=}0.5)$ over the $25$ feature dimensions, so
each client predominantly observes a different subset of channels.

Three independent sources of non-IID heterogeneity co-occur in our
federations (\Cref{fig:heterogeneity}): (i)~\emph{feature-distribution
skew}\,---\,per-client std of the per-feature mean spans $0.14$--$2.47$
on SMD, $0.011$--$1.96$ on SMAP and $0.015$--$1.33$ on
MSL\,---\,roughly two orders of magnitude across the federation;
(ii)~\emph{label-distribution skew}\,---\,the per-client anomaly
ratio spans $0.4$--$27\%$ on SMD and $0$--$98\%$ on SMAP/MSL, while
PSM is essentially uniform at ${\approx}30\%$ by construction; and
(iii)~\emph{quantity skew}\,---\,per-client training length varies
$1.2{\times}$ on SMD, ${\approx}10{\times}$ on SMAP/MSL and up to
${\approx}25{\times}$ on PSM.

All data are $z$-scored per client (SMD/SMAP/MSL: each client is its
own physical entity, zero leakage) or globally on PSM (a single
shared test stream forces a single scaler), then clipped to
$\pm 10\sigma$ to remove training-stream spikes that destabilise
reconstruction baselines. We use a sliding window of length $w{=}20$
with stride $1$, matching USAD/TranAD. The shared FedAvg
envelope (\Cref{tab:fl-config}) trains all five methods for $30$
global rounds with $25\%$ client sampling per round, $5$ local
epochs, \texttt{Adam} at $\text{lr}{=}1{\times}10^{-3}$, batch size
$128$, and a temporal $85\%/15\%$ per-client train/validation split.
For \textsc{FedKAD}, we use the same default hyperparameter setting on SMD, SMAP, and MSL: Koopman lifting dimension $d_{\mathrm{lift}}=128$, subspace rank $r=24$, ADMM penalty $\rho=1.0$, Stiefel-ADMM step size $0.05$, ridge regularisation $\lambda=10^{-2}$, spectral cap $0.995$, score-smoothing window of $5$ steps, and prediction horizon $H=1$. For PSM, all hyperparameters are kept unchanged except for the subspace rank, which is increased from $r=24$ to $r=32$ based on a held-out validation sweep without using test labels.
Every \texttt{(dataset, method)} configuration is repeated for three seeds
$\{0,1,2\}$; we report mean$\pm$std across seeds.
Implementation: Python~3.10, PyTorch~2.1.0, CUDA~12.1; experiments
run on $4{\times}$ NVIDIA A40 ($45$\,GB) GPUs.

\subsubsection{Baselines}
We compare \textsc{FedKAD} against four representative neural
baselines for MTAD: (1)~\textbf{DeepSVDD}~\cite{ruff2018deep}, a
one-class deep network with hypersphere objective;
(2)~\textbf{LSTM-AE}~\cite{malhotra2016lstm}, an autoencoder of
stacked LSTM units; (3)~\textbf{USAD}~\cite{audibert2020usad}, an
adversarial encoder--decoder for unsupervised MTAD; and
(4)~\textbf{TranAD}~\cite{tuli2022tranad}, a transformer-based
reconstructor with an adversarial decoder. Because these methods are
originally designed for centralised training, we wrap them in
\textbf{FedAvg}~\cite{mcmahan2017communication} with the identical
federated envelope used for \textsc{FedKAD}, so the only dimension of
variation is the local model class. LSTM-AE, USAD and TranAD use the
default architectures of the original papers; for DeepSVDD we disable
the autoencoder pretraining and centralised hypersphere-centre
initialisation under our FL-faithful evaluation because they would
require server access to the union of every client's raw training
data, violating the FL privacy constraint.



\begin{table}[!tbp]
\centering
\caption{Federated training configuration shared by all methods.}
\label{tab:fl-config}
\small
\resizebox{8.9cm}{!} {
\begin{tabular}{lclc}
\toprule
Component & Value & Component & Value \\
\midrule
Aggregation~\cite{mcmahan2017communication} & FedAvg
& Global rounds & $30$ \\
Local epochs / round & $5$
& Client sampling / round & $25\%$ \\
Sampling policy & uniform
& Local batch size & $128$ \\
Optimiser, lr & \texttt{Adam}, $10^{-3}$
& Train / val split & temporal $85\%/15\%$ \\
Window / stride & $20/1$
& Seeds & ${0,1,2}$ \\
\bottomrule
\end{tabular}
}
\end{table}


\subsubsection{Evaluation metrics}
We report Precision, Recall, F$_1$, and AUC under a per-client macro
aggregation protocol. Specifically, each metric is first computed separately
on each client's test stream using its local ground-truth labels, and the
resulting scores are then averaged uniformly across clients. This protocol
prevents clients with longer streams from dominating the aggregate result and
therefore reflects performance at the federation level rather than at the
pooled-sample level.

To facilitate comparison with prior work while avoiding over-reliance on the
known optimistic bias of point adjustment~\cite{kim2022rigorous}, we evaluate
four complementary thresholding and adjustment protocols. First, our primary
headline metric uses segment-aware PA\%K with $k{=}0.01$~\cite{kim2022rigorous},
where the detection threshold is tuned on each client's validation split.
Under PA\%K, an anomalous segment is counted as detected only if at least a
fraction $k$ of its points are predicted as anomalous; thus, $k{=}0.01$
corresponds to requiring at least $1\%$ of the segment to be detected. This
criterion mitigates the classic any-hit pathology of point adjustment, where
a single isolated alarm inside an anomalous segment is sufficient to mark the
entire segment as correctly detected. Second, we report the conventional
blind-threshold setting based on POT~\cite{siffer2017anomaly} with any-hit
point adjustment ($k{=}0$)~\cite{xu2018unsupervised}; here, the threshold is
selected without access to test labels and an anomalous segment is considered
detected if at least one point within the segment is flagged. Third, we
evaluate a stricter blind setting that combines POT thresholding with PA\%K
at $k{=}0.1$, corresponding to requiring at least $10\%$ of an anomalous
segment to be detected, together with a recall cap of $0.995$ to reject
lucky-spike thresholds. Finally, we report strict point-wise F$_1$ at the POT
threshold without point adjustment, which measures detection quality at the
individual timestamp level.
\subsection{Main results}
\label{sec:results-main}

\subsubsection{Detection performance and statistical significance}

\Cref{fig:f1-matrix-pa001} summarizes the per-client averaged F$_1$
scores of \textsc{FedKAD} and the four baselines under the primary
PA
on three of the four datasets, namely SMD, PSM, and SMAP, while showing
only a small degradation on MSL. On SMD, \textsc{FedKAD} obtains
$72.20{\pm}1.28$, outperforming the strongest baseline, LSTM-AE, by
$+7.69$ F$_1$ ($p{=}0.038$, $d_z{=}+2.86$). On PSM, \textsc{FedKAD}
achieves $69.87{\pm}0.21$, improving over LSTM-AE by $+1.29$ F$_1$
($p{=}0.0092$, $d_z{=}+5.97$). On SMAP, \textsc{FedKAD} reaches
$53.32{\pm}0.50$, giving a positive margin of $+1.76$ F$_1$ over
LSTM-AE. Although this gain is not statistically significant at the
$p{<}0.10$ level, it exceeds the predefined verdict margin of
$0.5$ F$_1$. On MSL, \textsc{FedKAD} obtains $39.55{\pm}0.02$, compared
with $40.50{\pm}4.63$ for LSTM-AE, corresponding to a small loss of
$0.95$ F$_1$. 


\begin{figure}[!tbp]
\centering
\includegraphics[width=0.7\columnwidth]{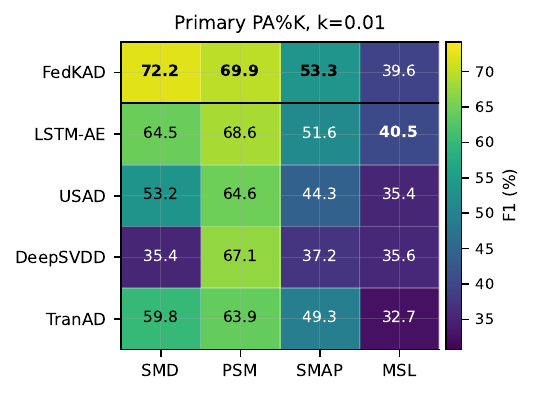}
\caption{Performance under the primary PA\%K protocol with $k=0.01$.}
\label{fig:f1-matrix-pa001}
\end{figure}

\begin{table}[!tbp]
\centering
\caption{Paired-$t$ test of \textsc{FedKAD} versus the strongest
baseline per dataset, with Cohen $d_z$ effect size and F$_1$ margin.
Verdict gate: $\text{margin}{>}0.5$~F$_1$.}
\label{tab:main-stats}
\small
 \resizebox{8.9cm}{!} {
\begin{tabular}{lccccc l}
\toprule
Dataset & \textsc{FedKAD} F$_1$ & Best baseline & Margin & $p$ & $d_z$ & Outcome \\
\midrule
SMD  & 72.20$\pm$1.28 & LSTM-AE ($64.51$) & $+7.69$ & $0.0384$ & $+2.86$ & \textbf{WIN} \\
PSM  & 69.87$\pm$0.21 & LSTM-AE ($68.57$) & $+1.29$ & $0.0092$ & $+5.97$ & \textbf{WIN} \\
SMAP & 53.32$\pm$0.50 & LSTM-AE ($51.56$) & $+1.76$ & $0.25$   & $+0.93$ & \textbf{WIN} \\
MSL  & 39.55$\pm$0.02 & LSTM-AE ($40.50$) & $-0.95$ & $0.80$   & $-0.17$ & loss \\
\bottomrule
\end{tabular}
}
\end{table}

\Cref{tab:main-stats} reports the paired-$t$ test between
\textsc{FedKAD} and the strongest baseline on each dataset. The results
confirm that the F$_1$ improvements on SMD and PSM are statistically
supported, with margins of $+7.69$ and $+1.29$ F$_1$, respectively, and
large effect sizes ($d_z{=}+2.86$ and $d_z{=}+5.97$). On SMAP,
\textsc{FedKAD} still improves over the strongest baseline by
$+1.76$ F$_1$, although the paired test is not significant because of
seed-level variability. On MSL, \textsc{FedKAD} is below LSTM-AE by only
$0.95$ F$_1$, with a negligible effect size ($d_z{=}-0.17$). It is easily to see that
\textsc{FedKAD} satisfies the predefined verdict gate on three of the
four datasets, while its only loss remains small.

\subsubsection{Robustness across evaluation protocols}

\Cref{tab:protocol-robustness}
evaluate FedKAD under three supplementary protocols. Under the classic
POT threshold with any-hit point adjustment, FedKAD remains highly
competitive and achieves the best F$_1$ on SMD and PSM, while tying
LSTM-AE on MSL. However, this protocol saturates recall on long anomaly
segments, so the ranking mainly reflects precision. When the PA\%K requirement is tightened to $k=0.10$, FedKAD keeps a clear advantage on SMD but becomes less competitive on PSM, SMAP, and
MSL. This indicates that FedKAD detects anomalous segments reliably on
server-machine streams but is more sensitive to stricter segment-coverage
requirements on spacecraft and PSM streams. Finally, under the strict point-wise no-PA protocol, all methods obtain substantially lower F$_1$ scores, confirming the difficulty of exact
point-level localization. FedKAD remains competitive on MSL but does not
dominate this strict setting. Overall, the multi-protocol results show
that FedKAD's main advantage is strongest under segment-level anomaly
detection protocols.

\begin{figure*}[!tbp]
\centering
\includegraphics[width=0.85\linewidth,height=0.30\textheight,keepaspectratio]{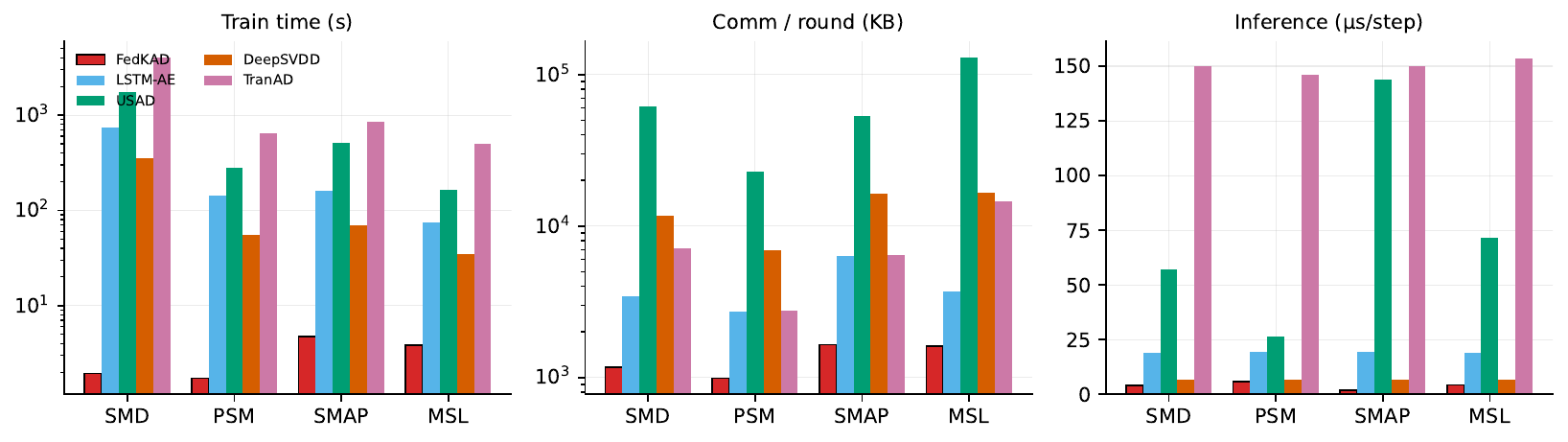}
\caption{Federated efficiency: training wall-clock (s), per-round
communication (KB) and inference latency ($\mu$s/step), log scale.
\textsc{FedKAD} is two to three orders of magnitude faster to train
and one to two orders cheaper in communication than the next-best
neural baseline, with comparable or lower inference latency.}
\label{fig:efficiency}
\end{figure*}

\begin{table}[t]
\centering
\caption{FedKAD robustness under supplementary evaluation protocols. 
The best baseline is selected per dataset and protocol.}
\label{tab:protocol-robustness}
\resizebox{\columnwidth}{!}{
\begin{tabular}{llccc}
\toprule
Protocol & Dataset & FedKAD F$_1$ & Best baseline F$_1$ & Margin \\
\midrule
\multirow{4}{*}{POT + any-hit PA}
& SMD  & \textbf{89.19$\pm$0.78} & LSTM-AE 83.94$\pm$0.67 & +5.25 \\
& PSM  & \textbf{98.88$\pm$0.19} & LSTM-AE 98.37$\pm$0.15 & +0.51 \\
& SMAP & 94.76$\pm$0.02 & USAD \textbf{96.17$\pm$0.96} & -1.41 \\
& MSL  & \textbf{83.59$\pm$0.08} & LSTM-AE 83.58$\pm$0.15 & +0.02 \\
\midrule
\multirow{4}{*}{PA\%K, $k=0.10$}
& SMD  & \textbf{55.30$\pm$2.39} & LSTM-AE 44.68$\pm$0.59 & +10.62 \\
& PSM  & 60.58$\pm$0.42 & LSTM-AE \textbf{67.15$\pm$0.11} & -6.57 \\
& SMAP & 37.28$\pm$1.33 & TranAD \textbf{44.78$\pm$0.18} & -7.50 \\
& MSL  & 41.76$\pm$0.01 & LSTM-AE \textbf{42.94$\pm$1.54} & -1.17 \\
\midrule
\multirow{4}{*}{Point-wise no-PA}
& SMD  & 14.61$\pm$0.87 & USAD \textbf{17.30$\pm$0.64} & -2.69 \\
& PSM  & 24.57$\pm$0.06 & DeepSVDD \textbf{28.65$\pm$8.68} & -4.08 \\
& SMAP & 3.87$\pm$0.01 & DeepSVDD \textbf{15.60$\pm$2.36} & -11.73 \\
& MSL  & \textbf{7.50$\pm$0.02} & DeepSVDD 6.91$\pm$0.72 & +0.60 \\
\bottomrule
\end{tabular}
}
\end{table}

\subsubsection{Federated efficiency}
\Cref{fig:efficiency}
reports end-to-end
training wall-clock, per-round communication and inference latency.
On all four datasets, \textsc{FedKAD} trains in $1.7$--$4.8$\,s
versus $75$--$4020$\,s for the neural baselines (a
$200{\times}$--$2000{\times}$ speed-up), exchanges $1.0$--$1.7$\,MB
per round versus $2.7$--$129$\,MB (a $3{\times}$--$40{\times}$
reduction), and incurs $1.9$--$5.8$\,$\mu$s per inference step, comparable to or below every neural baseline. Both savings
stem from \textsc{FedKAD}'s two architectural commitments:
closed-form local fitting, which avoids local-SGD drift under non-IID data,
and low-rank subspace consensus, where communication scales with the subspace
size $O(Lr)$ rather than the number of trainable model parameters
$O(|\theta|)$. In a federated deployment, TranAD's
hour-long per-client round and USAD's $60$--$130$\,MB per-round
upload are prohibitive on edge hardware, whereas \textsc{FedKAD}
fits in single-second rounds with sub-$2$\,MB uplinks. Combining the
accuracy verdict of \Cref{tab:main-stats} with the efficiency
profile of \Cref{fig:efficiency}, \textsc{FedKAD} is the only
method in our study that is simultaneously accurate (on $3/4$
datasets) and FL-deployable.

\subsubsection{On-device deployment (Raspberry Pi~4)}
To complement the data-centre measurements of \Cref{fig:efficiency},
we re-evaluate every method end-to-end on commodity edge hardware:
a Raspberry Pi~4 (Broadcom BCM2711, $4$\,$\times$ Cortex-A72 @
$1.5$\,GHz, ARMv8, $4$\,GB LPDDR4) running $64$-bit Raspberry Pi OS,
Python~3.13 and PyTorch~2.12 (CPU EP). The setting mirrors SMAP
($n_{\text{features}}{=}25$, $w{=}20$, FedKAD $r{=}24$,
$d_{\mathrm{lift}}{=}128$), with $100$ warm-started repeats per measurement and the
median reported. \Cref{fig:efficiency-pi} plots training time per global round
(s, $\log$), inference latency ($\mu$s/step) and per-round
communication: the only quantity uploaded each FedAvg round is the
Stiefel basis $P_i \in \mathbb{R}^{D\times r}$, so for FedKAD the
bubble area encodes $|P_i|{=}59$\,KB, an order of magnitude below
LSTM-AE ($227$\,KB) and USAD ($1.9$\,MB).
\textsc{FedKAD} reaches the lower-left efficient corner on the Pi:
$0.23$\,s/round training (vs.\ $0.32$\,s for the next-best
DeepSVDD, $0.87$\,s for LSTM-AE, and $5.66$\,s for TranAD) and
$0.79$\,$\mu$s/step inference (vs.\ $2.03$\,$\mu$s for DeepSVDD and
$24.5$\,$\mu$s for TranAD)\,---\,a $24{\times}$ training speed-up and
$31{\times}$ inference speed-up over the heaviest baseline, on a
\$$35$ device. Combined with the $59$\,KB per-round upload, this
establishes that \textsc{FedKAD} is the only method in our study
that fits the compute, memory and bandwidth budget of a real edge
deployment.

\begin{figure*}[!thp]
\centering
\includegraphics[width=0.95\linewidth,height=0.25\textheight,keepaspectratio]{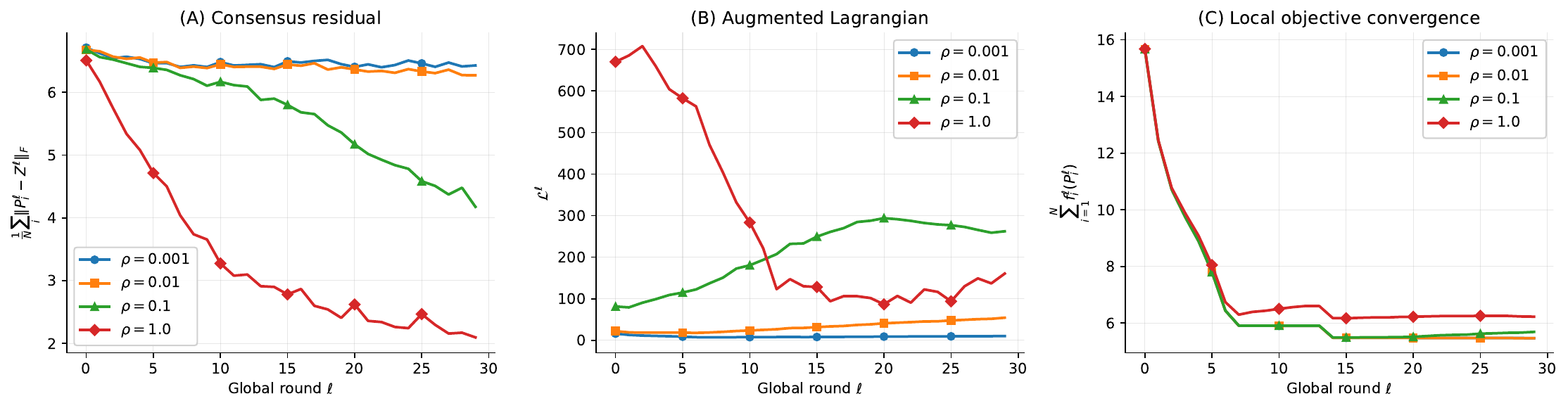}
\caption{Empirical convergence diagnostics on SMD (\(N{=}28\) clients,
\(r{=}24\), \(30\) rounds). \textbf{(A)} Mean consensus residual
\(\frac{1}{N}\sum_i\|P_i^\ell{-}Z^\ell\|_F\): decays at
\(\rho{\geq}1\) (\Cref{thm:fedkad_lagrangian_convergence}). \textbf{(B)} Augmented
Lagrangian \(\mathcal{L}^{\ell}\): descends toward a bounded
neighbourhood at \(\rho{=}1\); the late-round oscillation is consistent
with \Cref{lemma:fedkad_descent} under partial client participation and
the \(\Delta_Q^\ell\) perturbation from the \(Q_i\)-refit (A5).
\textbf{(C)} Local objective
\(\sum_i f_i^\ell(P_i^\ell)\): decreases over training.}
\label{fig:conv-theory}
\end{figure*}

\begin{figure}[!tbp]
\centering
\includegraphics[width=\columnwidth,keepaspectratio]{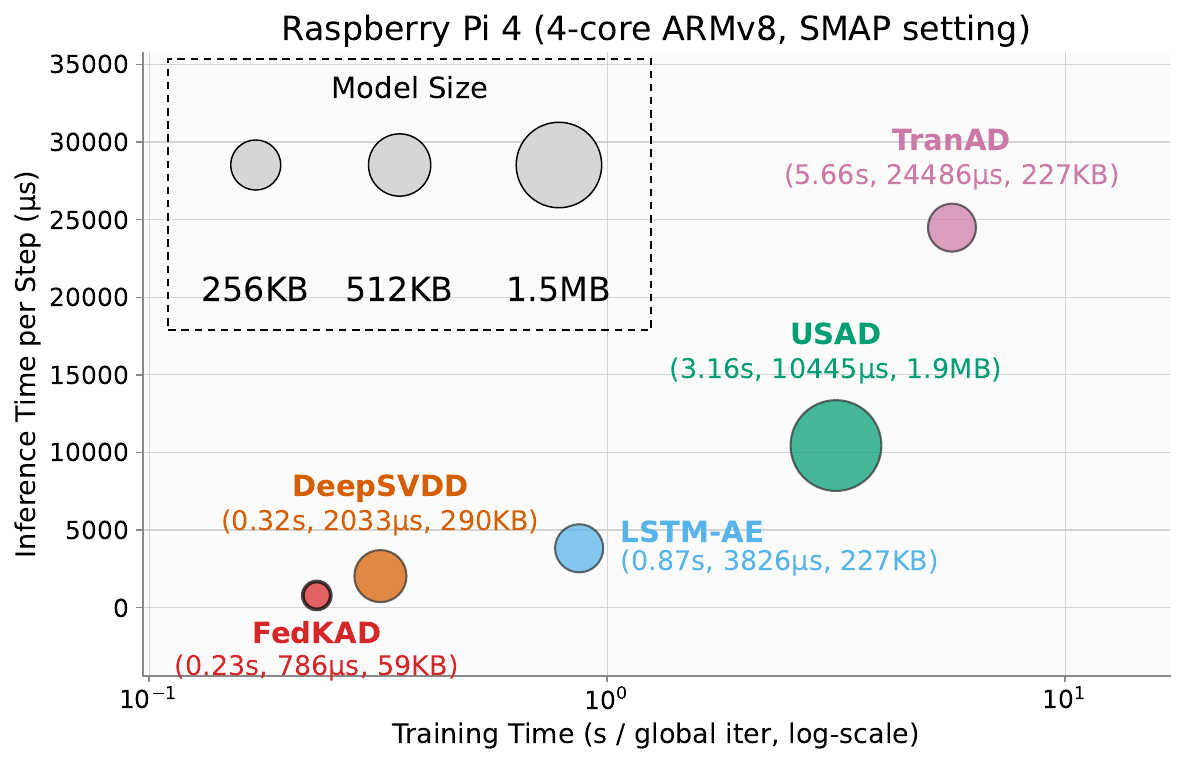}
\caption{On-device benchmark on a Raspberry Pi 4 under the SMAP setting. \textsc{FedKAD} achieves the best edge-deployment trade-off, requiring only 59 KB upload per round while reducing training and inference time by up to 25$\times$ and 31$\times$ over the slowest baseline.}
\label{fig:efficiency-pi}
\end{figure}

\subsubsection{Convergence guarantee (theoretical companion)}
\label{par:convergence-guarantee}
The federated training procedure of \textsc{FedKAD} is cast as a
non-convex consensus ADMM on the Stiefel manifold and admits the
convergence guarantee proved in \Cref{thm:fedkad_lagrangian_convergence} shows that for any
sufficiently large penalty \(\rho\) satisfying
\(\rho\mu_i(\rho)>2L_i^2\), the augmented Lagrangian
\(\{\mathcal{L}^{\ell}\}\) converges to a finite value, the successive
differences \(\|P_i^{\ell+1}-P_i^\ell\|_F\) and
\(\|Z^{\ell+1}-Z^\ell\|_F\) vanish, and the consensus residual
\(\|P_i^{\ell+1}-Z^{\ell+1}\|_F\to0\) for every client. The empirical convergence behaviour is further examined
below by tracking the consensus residual, augmented Lagrangian, and
local data-fit term under different choices of \(\rho\).

\Cref{fig:conv-theory} visualises the three convergence indicators on
SMD (\(N{=}28\) clients) for
\(\rho\in\{10^{-3},10^{-2},10^{-1},1\}\). At \(\rho{=}1\) (the default
operating point), the mean consensus residual
\(\frac{1}{N}\sum_i\|P_i^\ell-Z^\ell\|_F\) drops from \(6.5\) to \(2.1\)
over \(30\) rounds (panel~A), the augmented Lagrangian
\(\mathcal{L}^{\ell}\) descends from \(670\) to a neighbourhood of
\({\approx}90\)--\(160\) (panel~B), and the local data-fit term
\(\sum_i f_i^\ell(P_i^\ell)\) decreases from \(15.7\) to \(6.2\)
(panel~C). The Lagrangian oscillates mildly in later rounds rather than
descending strictly; this is consistent with
\Cref{lemma:fedkad_descent}, which permits per-round increases when the
finite-inner-loop error \(\varepsilon_\ell\), the Koopman refitting
perturbation \(\Delta_Q^\ell\), or partial-participation effects exceed
the descent terms, while still guaranteeing convergence to a finite
limit under A2 and A5. At smaller \(\rho\), the penalty is insufficient:
the consensus residual stagnates and \(\mathcal{L}^{\ell}\) grows
because the dual accumulation exceeds the descent. The transition occurs
between \(\rho{=}0.1\) and \(\rho{=}1\), confirming the
``sufficiently large \(\rho\)'' condition of the theorem and validating
the default choice.

%% file: sections/conclusion/conclusion.tex
\section{Conclusion}
This paper presented \textsc{FedKAD}, a federated low-rank Koopman framework for anomaly detection in resource-constrained IoT multivariate time series. By learning compact Koopman dynamics and exchanging only low-rank subspace variables, \textsc{FedKAD} avoids the heavy training and communication cost of federated neural baselines while keeping raw data and local dynamics on device. Experiments on four benchmarks show that \textsc{FedKAD} achieves the best F$_1$ score on three datasets under the primary PA\%K protocol and provides substantial reductions in training time, communication, and inference latency. The on-device Raspberry Pi 4 benchmark further demonstrates its practicality for commodity edge hardware with limited compute, memory, and bandwidth. Future work will study adaptive rank selection, personalized thresholding, and stronger privacy mechanisms.